\documentclass[10pt]{article}

\usepackage[numbers,sort&compress]{natbib}

\usepackage{graphicx}
\graphicspath{{figure/}{figure/intro/}{figure/experiments/}{figure/toy_example/}}
\usepackage{subcaption}
\usepackage{float}

\usepackage{booktabs}
\usepackage{multirow}
\usepackage{makecell}

\usepackage{amsmath}
\usepackage{bm}
\usepackage{amsthm}
\usepackage{amsfonts}
\usepackage{amssymb}
\usepackage{mathtools}
\usepackage{xparse}
\usepackage{xspace}
\mathtoolsset{showonlyrefs=true}
\newtheorem{theorem}{Theorem}[section]
\newtheorem{definition}{Definition}[section]
\newtheorem{lemma}{Lemma}[section]

\newtheorem{proposition}{Proposition}[section]

\usepackage[
    vlined,
    ruled,
    linesnumbered
]{algorithm2e}

\newcommand{\R}{\mathbb{R}}
\newcommand{\Prob}{\mathbb{P}}
\newcommand{\proposed}{\texttt{post-ADC}\xspace}
\newcommand{\withouteta}{\texttt{w/o-}$\eta$\xspace}
\newcommand{\withoutT}{\texttt{w/o-}$\mathcal{T}$\xspace}
\newcommand{\bonferroni}{\texttt{Bonferroni}\xspace}
\newcommand{\naive}{\texttt{naive}\xspace}

\usepackage[hidelinks]{hyperref}

\title{\vspace{-3cm}Post-ADC Inference:\\ Valid Inference After Active Data Collection}

\date{\today}

\makeatletter
\def\@fnsymbol#1{\ensuremath{\ifcase#1\or
{1}\or
{\ast}\or
{2}\or
{\dagger}\or
\else\@ctrerr\fi}}
\makeatother

\author{
Shuichi Nishino\thanks{Nagoya University} \thanks{RIKEN},
Tomohiro Shiraishi\footnotemark[1] \footnotemark[2],
Teruyuki Katsuoka\footnotemark[1],\\
Ichiro Takeuchi\footnotemark[1] \footnotemark[2]  \thanks{Corresponding author. e-mail: takeuchi.ichiro.n6@f.mail.nagoya-u.ac.jp}
}

\begin{document}

\maketitle
\thispagestyle{empty}

\begin{abstract}
    \noindent
    The validity of statistical inference depends critically on how data are collected.
When data gathered through active data collection (ADC) are reused for a post-hoc inferential task, conventional inference can fail because the sampling is adaptively biased toward regions favored by the collection strategy.
This issue is especially pronounced in black-box optimization, where sequential model-based optimization (SMBO) methods such as the tree-structured Parzen estimator (TPE) and Gaussian process upper confidence bound (GP-UCB) preferentially concentrate evaluations in promising regions.
We study statistical inference on actively collected data when the inferential target is constructed in a data-dependent manner after data collection.
To enable valid inference in this setting, we propose post-ADC inference, a framework that accounts for the biases arising from both the active data collection process and the subsequent data-driven target construction.
Our method builds on selective inference and provides valid $p$-values and confidence intervals that correct for both sources of bias.
The framework applies to a broad class of ADC processes by imposing only assumptions on the observation noise, without requiring any assumptions on the underlying black-box function or the surrogate model used by the SMBO algorithm.
Empirical results also show that post-ADC inference provides valid inference for data collected by GP-UCB and TPE.

\end{abstract}

\newpage
\section{Introduction}
\label{sec:intro}

Quantifying the statistical significance of insights obtained from data fundamentally depends on how the data are collected.
In principle, valid statistical inference, such as the computation of $p$-values or confidence intervals, requires that data be sampled through a procedure aligned with the inferential objective.
However, in practice, datasets collected for one task are often reused for another.
When the data collection process is biased with respect to the new task, the resulting statistical conclusions may be invalid.
This issue becomes particularly critical in settings involving active data collection (ADC).
Many statistical methods assume that data are independently and identically distributed (i.i.d.), but this assumption is violated when data are collected adaptively based on past observations.
We refer to this setting as post-ADC analysis.

As a proof-of-concept (PoC) of post-ADC analysis, we study black-box optimization problems by sequential model-based optimization (SMBO) methods~\cite{jones1998efficient,srinivas2010gaussian,bergstra2011algorithms,hutter2011sequential,frazier2018tutorial}.
SMBO is a widely used approach for black-box optimization, in which a surrogate model of the objective function is iteratively updated and used to actively select evaluation points.
It has been successfully applied to hyperparameter tuning in machine learning, as well as to optimization problems across various scientific domains.
Representative methods include density-based methods such as the tree-structured Parzen estimator (TPE)~\citep{bergstra2011algorithms} and Gaussian process-based approaches such as Gaussian process upper confidence bound (GP-UCB)~\citep{srinivas2010gaussian}.
A common characteristic of these methods is that the collected data are highly biased.
For example, in maximization problems, samples tend to concentrate in regions with high observed values, while regions with low values are underexplored.
Consequently, the dataset is not an i.i.d. sample from the underlying distribution, but rather the outcome of a goal-directed adaptive process.

This bias poses a fundamental challenge when such data are used for a post-hoc inferential task.
The difficulty is compounded when the inferential target is itself constructed in a data-driven manner from the same adaptively collected data.
For example, after an ADC process, one may wish to formulate a new statistical question about the underlying black-box function, such as a two-sample test comparing ``high'' and ``low'' regions.
However, because the sampling process depends on past observations and the inferential target is defined after observing the resulting dataset, naive statistical inference can yield misleading conclusions, resulting in overconfident or spurious findings.
Figure~\ref{fig:intro_demonstration} provides a concrete illustration of this setting.

\begin{figure}[h]
 \centering
 \includegraphics[width=0.49\textwidth,trim=0.2cm 0.2cm 0.1cm 1.2cm, clip]{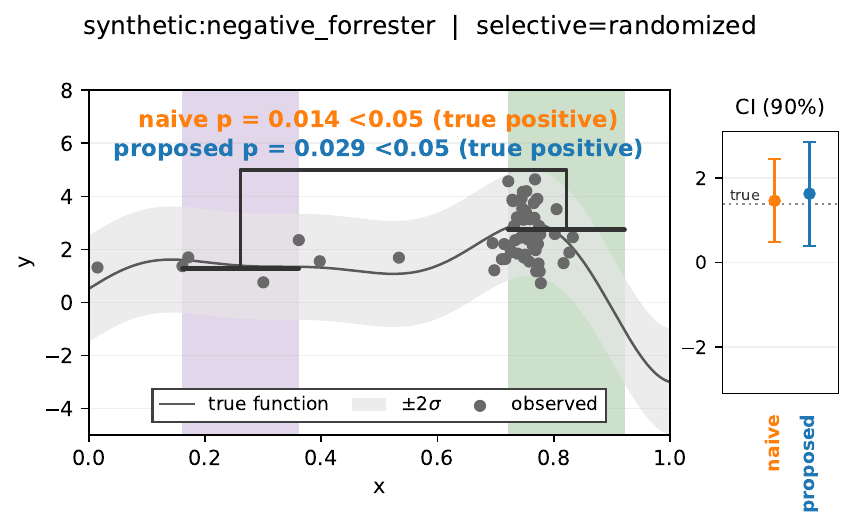}
 \includegraphics[width=0.49\textwidth,trim=0.1cm 0.2cm 0.2cm 1.2cm, clip]{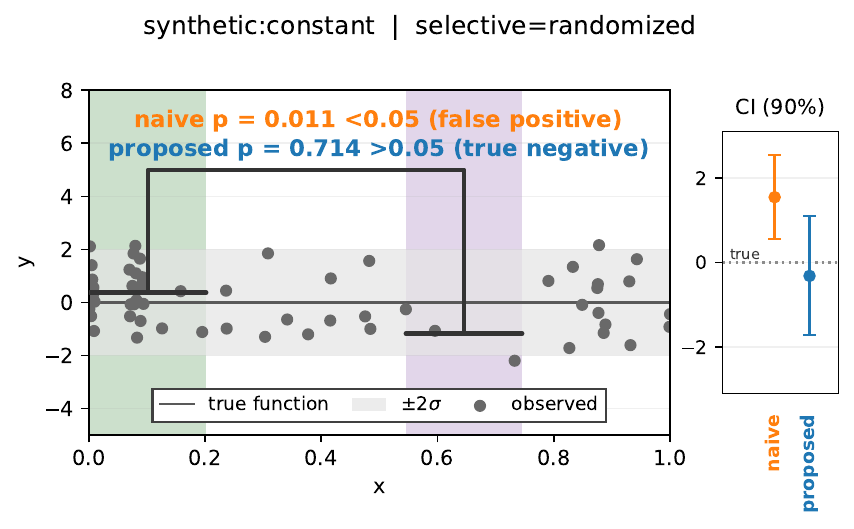}
 \vspace{-0.75em}
 \caption{
 Schematic illustration of a post-ADC analysis.
 The left panel shows the case where the true black-box function depends on $x$, indicating underlying differences between regions.
 In contrast, the right panel shows the case where it does not.
 Dots indicate data points collected through an SMBO process.
 Based on observed data, regions with high (green) and low (purple) estimated means are selected, and statistical inference ($p$-values and confidence intervals) is performed to compare the two regions.
 Naive inference methods can lead to false-positive findings and miscalibrated confidence intervals, whereas the proposed inference approach provides valid inference.
}
\label{fig:intro_demonstration}
\end{figure}

The goal of this paper is to develop a principled framework for valid statistical inference in post-ADC analysis.
Our approach is based on selective inference (SI), which enables valid inference for data-dependent inferential targets by conditioning on the selection event~\citep{benjamini2009selective, fithian2014optimal, lee2016exact, tibshirani2016exact}.
We refer to this framework as \textit{post-ADC inference}.
In SI, inference is performed not by treating the target as fixed in advance, but by conditioning on the fact that it was selected after observing the data.
The key novelty of our setting is that selection occurs at two levels: not only is the inferential target constructed from the data, but the data themselves are collected through an adaptive process.
We therefore treat the entire pipeline—from the ADC process to the data-driven construction of the inferential target—as a single selection event and perform inference conditional on it.
This allows us to properly account for the bias induced by ADC.
We show that, for a broad class of ADC algorithms, including GP-UCB and TPE, the conditional distribution of the test statistic can be characterized in a closed form.
This enables the construction of valid $p$-values that control type-I error and confidence intervals with valid coverage, even when the data are reused for different tasks after ADC.
The proposed post-ADC inference provides exact finite-sample guarantees of statistical significance and coverage without requiring any assumptions about the underlying black-box function or the quality of the surrogate model, except for assumptions on the observation noise.

\paragraph{Related work.}
SMBO is a standard framework for black-box optimization, in which a surrogate model is fitted to past observations and the next evaluation point is selected based on its predictions.
Depending on the design of the surrogate model, SMBO methods can be broadly classified into three categories~\cite{jones1998efficient,srinivas2010gaussian,bergstra2011algorithms,hutter2011sequential,frazier2018tutorial}.
The first category consists of Gaussian process–based methods, which model the objective function with a Gaussian process (GP) surrogate and use acquisition functions, such as expected improvement and GP-UCB, to select evaluation points based on predictive means and uncertainty estimates~\cite{jones1998efficient,srinivas2010gaussian}.
The second category consists of density estimation-based methods such as TPE, which use density ratios between high- and low-performing regions and work well with conditional or discrete search spaces~\cite{bergstra2011algorithms}.
The third category consists of ensemble-based methods such as sequential model-based algorithm configuration (SMAC), which use random forest surrogates and are robust in nonlinear mixed-type spaces~\cite{hutter2011sequential}.
Despite these differences, all concentrate evaluations in regions that past observations make promising, so the collected data are typically not i.i.d. but biased.
Existing SMBO research has therefore emphasized optimization performance, sample efficiency, and regret, while the validity of post hoc inference from actively collected data remains underexplored.

SI is a framework for valid statistical inference after hypotheses or targets are selected from the data~\cite{benjamini2009selective,fithian2014optimal,lee2016exact,tibshirani2016exact}.
Early work focused on variable selection in linear regression, including Lasso, marginal screening, and stepwise selection~\cite{fithian2014optimal,lee2016exact,tibshirani2016exact,taylor2018post,tian2018selective,tian2017asymptotics,yang2016selective,liu2018more,hyun2018exact,benjamini2020selective,taylor2014post,fithian2015selective}.
It has since been extended to clustering, change-point detection, and image or deep learning interpretation, yielding error-rate and coverage guarantees for data-driven targets~\cite{gao2022selective,duy2022quantifying,miwa2023valid,shiraishi2025statistical,matsukawa2025statistical,markovic2017unifying}.
Most SI studies, however, assume a fixed sampling mechanism and address selection only after the dataset is observed.
Related work on adaptive experiments and bandits treats sequentially changing data collection, for example in treatment assignment or arm selection~\cite{nie2018adaptive_bias,freidling2024sri,chen2023optimal}.
For post-ADC analysis, where the surrogate model and acquisition function make the sampling rule highly data-dependent, a principled SI framework remains largely unexplored.
In this study, we propose an SI framework that treats the ADC process and the construction of post hoc inferential targets in a unified manner.

\paragraph{Contributions.}
Our contributions are threefold.
First, in \S~\ref{sec:selective_pvalues}, we formulate post-ADC inference by conditioning on a unified selection event that combines the ADC process and the data-driven construction of the inferential target, thereby yielding valid selective $p$-values and confidence intervals.
Second, in \S~\ref{sec:exact_computation}, we show that, for a broad class of ADC processes, the conditional distribution of the test statistic admits a closed-form expression, enabling exact finite-sample computation of these quantities.
Third, as PoC studies, we consider the post-ADC setting for data collected by GP-UCB and TPE, and empirically demonstrate that the post-ADC inference produces valid inferential results.
For reproducibility, we provide the implementation code at \url{https://github.com/ni-shu/paper_post-adc-inference}.

\newpage
\section{Active Data Collection}
\label{sec:problem_setup}

We consider ADC on a finite candidate set
\(
\mathcal{X}=\{\bm{x}^{(1)},\ldots,\bm{x}^{(M)}\}\subset \R^d
\).
At step \(n \in [N-1] := \{1, 2, \ldots, N-1\}\), the algorithm has access to the dataset
\(
\mathcal D_n=((\bm x_1,y_1),\ldots,(\bm x_n,y_n))
\in
(\mathcal X\times\mathbb R)^n .
\)
An ADC algorithm is represented by a sequence of update rules
\(
\mathcal A^{\rm ADC}:(\mathcal X\times\mathbb R)^n\to\mathcal X,
\)
and the next query point \(\bm x_{n+1}\) is given by
\begin{equation}
\bm x_{n+1}
=
\mathcal A^{\rm ADC}(\mathcal D_n).
\label{eq:adc_map}
\end{equation}
Often, this update rule is expressed as the maximization of an acquisition score $a(\bm{x}; \mathcal{D}_n)$:
\begin{equation}
  \bm x_{n+1} =
  \arg\max_{\bm{x}\in\mathcal{X}} a(\bm{x};\mathcal{D}_n).
  \label{eq:adc_acquisition_maximization}
\end{equation}
The following TPE and GP-UCB are representative examples of ADC algorithms that arise in this form.
For notational simplicity, we write
\(\bm x_{1:n}:=(\bm x_1,\ldots,\bm x_n)\)
and
\(\bm y_{1:n}:=(y_1,\ldots,y_n)^\top\).

\paragraph{Example~1: TPE.}

Fix a quantile level \(\gamma\in(0,1)\) and set \(m_n=\lceil \gamma n\rceil\). For \(\bm{y}_{1:n}\), let \(\mathcal{H}_n(\bm{y}_{1:n})\subset[n]\) denote the indices of the top \(m_n\) responses and let \(\mathcal{L}_n(\bm{y}_{1:n})=[n]\setminus \mathcal{H}_n(\bm{y}_{1:n})\) denote the remaining indices.
Given a positive kernel
\(
\kappa_{\mathrm{TPE}}:\mathcal{X}\times\mathcal{X}\to(0,\infty)
\),
define
\begin{equation}
  \begin{aligned}
  g_n(\bm{x})
  &=
  \frac{1}{|\mathcal{H}_n(\bm{y}_{1:n})|}
  \sum_{i\in\mathcal{H}_n(\bm{y}_{1:n})}
  \kappa_{\mathrm{TPE}}(\bm{x},\bm{x}_i),\\
  \ell_n(\bm{x})
  &=
  \frac{1}{|\mathcal{L}_n(\bm{y}_{1:n})|}
  \sum_{i\in\mathcal{L}_n(\bm{y}_{1:n})}
  \kappa_{\mathrm{TPE}}(\bm{x},\bm{x}_i).
  \end{aligned}
  \label{eq:tpe_density_estimates}
\end{equation}
Then, the TPE acquisition score is
\begin{equation}
  a^{\mathrm{TPE}}(\bm{x};\mathcal{D}_n)
  =
  \frac{g_n(\bm{x})}{\ell_n(\bm{x})}.
  \label{eq:tpe_score}
\end{equation}
The corresponding ADC algorithm \(\mathcal A^{\mathrm{TPE}}\) is any deterministic maximizer of \(a^{\mathrm{TPE}}(\bm{x};\mathcal{D}_n)\) over \(\mathcal X\).

\paragraph{Example~2: GP-UCB.}
Suppose the optimizer uses a Gaussian-process surrogate with kernel
\(
k:\mathcal{X}\times\mathcal{X}\to\R
\).
Then the surrogate posterior mean and variance are
\begin{equation}
  \begin{aligned}
  \mu_n(\bm{x};\bm{y}_{1:n})
  &=
  \bm{k}_n(\bm{x})^\top
  (K_n+\sigma^2 I_n)^{-1}\bm{y}_{1:n},\\
  s_n^2(\bm{x})
  &=
  k(\bm{x},\bm{x})
  -
  \bm{k}_n(\bm{x})^\top
  (K_n+\sigma^2 I_n)^{-1}
  \bm{k}_n(\bm{x}).
  \end{aligned}
  \label{eq:gp_posterior}
\end{equation}
where
\(K_n\in\R^{n\times n}\) has entries
\([K_n]_{ij}=k(\bm{x}_i,\bm{x}_j)\), \(\sigma^2\) is the response-noise variance, and
\[
  \bm{k}_n(\bm{x})
  =
  (k(\bm{x},\bm{x}_1),\ldots,k(\bm{x},\bm{x}_n))^\top
  \in \R^n .
\]
Under this surrogate, the GP-UCB acquisition score is
\begin{equation}
  a^{\mathrm{GP\mathchar`-UCB}}(\bm{x};\mathcal{D}_n)
  =
  \mu_n(\bm{x};\bm{y}_{1:n})+\kappa s_n(\bm{x}),
  \label{eq:GPUCB}
\end{equation}
where \(\kappa>0\) controls the exploration-exploitation tradeoff.
The corresponding ADC algorithm \(\mathcal A^{\mathrm{GP\mathchar`-UCB}}\) is any deterministic maximizer of \(a^{\mathrm{GP\mathchar`-UCB}}(\bm{x};\mathcal{D}_n)\) over \(\mathcal X\).

Examples~1 and 2 illustrate representative ADC algorithms covered by our framework.
More generally, the framework applies to a broad class of ADC algorithms, which is formally characterized in Appendix~\ref{app:admissibility_framework}.
Roughly speaking, it covers algorithms whose acquisition rules can be decomposed into a sequence of piecewise-linear operations.
For rules such as \eqref{eq:tpe_score} and \eqref{eq:GPUCB}, this condition is satisfied whenever comparisons between candidate query points can be characterized by linear inequalities in the responses \(\bm{y}_{1:n}\).
Consequently, the framework applies to many SMBO algorithms, including the TPE and GP-UCB methods considered here.

We fix a collection budget \(N\) and run the ADC algorithm for \(N\) queries, producing the collected dataset \(\mathcal D_N\).
Our goal is not to analyze or improve those ADC methods themselves.
Rather, our goal is to carry out statistically valid inference that properly accounts for the data collection bias created by those ADC mechanisms.

\newpage
\section{Target of Inference and Difficulty of Naive Inference}
\label{sec:target_of_inference}

To formulate post-ADC inference, we now introduce the data model and the target of inference.

\subsection{Data Model for Post-ADC Inference}
\label{sec:data_model}

Let \(f_\star:\mathcal{X}\to\R\) be an unknown non-random objective function.
At ADC step \(i\), the algorithm queries a point \(\bm{x}_i\in\mathcal{X}\) and observes
\begin{equation}
  Y_i = f_\star(\bm{x}_i) + \varepsilon_i,
  \qquad
  \varepsilon_i \overset{\mathrm{i.i.d.}}{\sim} \mathcal{N}(0,\sigma^2),
  \label{eq:obs_model}
\end{equation}
where the noise variance \(\sigma^2>0\) is assumed to be either known or estimable from an independent dataset.
Importantly, the validity results below rely only on the noise model in \eqref{eq:obs_model} and do not require any assumption on $f_\star$ or on the accuracy of the surrogate model used by the ADC algorithm.

We regard the responses $y_1,\ldots,y_N$ observed during ADC in \S~\ref{sec:problem_setup} as realizations of \eqref{eq:obs_model}.
In the following, we use uppercase notation for random variables, $\bm{Y}_{1:n}=(Y_1,\ldots,Y_n)^\top$, and lowercase notation for realizations, $\bm{y}_{1:n}=(y_1,\ldots,y_n)^\top$.
Similarly, we write the dataset up to step $n$ as $\mathcal{D}_n(\bm{Y}_{1:n})= ((\bm x_1,Y_1),\ldots,(\bm x_n,Y_n))$ when random, and as $\mathcal{D}_n(\bm{y}_{1:n}) = ((\bm x_1,y_1),\ldots,(\bm x_n,y_n))$ when realized.
Furthermore, for notational simplicity, let $\bm{Y}=\bm{Y}_{1:N}\in\R^N$ denote the full response vector after ADC, and let $\bm{y}=\bm{y}_{1:N}\in\R^N$ denote its realization.

\subsection{Target of Inference and Validity Goals}
\label{subsec:target_of_inference}

We now define a post-ADC target of inference based on the collected data $\mathcal{D}_N(\bm{Y})$.
First, we introduce representative examples of data-driven targets in Examples~1 and 2.

\paragraph{Example~1: Two-sample inference for high versus low region.}
This example is illustrated in Figure~\ref{fig:intro_demonstration}.
The high and low regions are defined as the axis-aligned hypercubes of side length \(\ell\) with the largest and smallest average observed responses, respectively:
\begin{equation}
  \mathcal{I}_{\bm{Y}} \in \operatorname*{arg\,max}_{\mathcal{C}(\bm{s};\ell)|\mathcal{C}(\bm{s};\ell)\neq\emptyset, \bm{s}\in\R^d} \bar Y(\mathcal{C}),
  \qquad
  \mathcal{J}_{\bm{Y}} \in \operatorname*{arg\,min}_{\mathcal{C}(\bm{s};\ell)|\mathcal{C}(\bm{s};\ell)\neq\emptyset, \bm{s}\in\R^d} \bar Y(\mathcal{C}),
  \label{eq:region_selector}
\end{equation}
where $\ell>0$ is the side length of the sliding window, $\bm{s}\in\R^d$ is a sliding-window anchor,
\(
\mathcal{C}(\bm{s};\ell)=\{t\in[N]:\ \bm{x}_t \in \prod_{j=1}^d [s_j,s_j+\ell]\}
\)
is the sliding-window index set, \(\bar Y(\mathcal{C})=|\mathcal{C}|^{-1}\sum_{t\in\mathcal{C}} Y_t\) is the sample mean over the sliding window \(\mathcal{C}\), and \(\emptyset\) is the empty set.
The target of inference is then the true mean difference between these two regions:
\begin{equation}
  \Delta_{\bm{Y}}
  =
  \frac{1}{|\mathcal{I}_{\bm{Y}}|}
  \sum_{t\in\mathcal{I}_{\bm{Y}}} f_\star(\bm{x}_t)
  -
  \frac{1}{|\mathcal{J}_{\bm{Y}}|}
  \sum_{t\in\mathcal{J}_{\bm{Y}}} f_\star(\bm{x}_t).
  \label{eq:selected_effect_sum}
\end{equation}
This problem is similar to a standard two-sample test, but differs in that the comparison target is selected in a data-driven manner after the ADC process.
For example, in applications such as scientific parameter optimization, this target corresponds to the true performance gap between high- and low-performance regions in the parameter space.

\paragraph{Example~2: One-sample inference for top-$n$ performance.}
For a given integer \(n\in[N]\), let
\[
  \mathcal I_{\bm Y}^{\mathrm{top},n}
  \in
  \operatorname*{arg\,max}_{\mathcal S\subset[N]:\ |\mathcal S|=n}
  \sum_{t\in\mathcal S} Y_t .
\]
The target of inference is the average true response over the selected
top-\(n\) queried points:
\[
  \Delta^{\mathrm{top}}_{\bm Y}
  =
  \frac{1}{n}
  \sum_{t\in\mathcal I_{\bm Y}^{\mathrm{top},n}} f_\star(\bm x_t).
\]
This problem can be viewed as a data-driven one-sample inference problem.
For example, in applications such as hyperparameter optimization of machine learning models, this target corresponds to the true performance of the top-\(n\) parameter configurations found by ADC.

Both Examples~1 and~2 are special cases of a general class of data-dependent targets of the form
\begin{equation}
  \Delta_{\bm{Y}}
  =
  \bm{\eta}_{\bm{Y}}^\top \bm{\mu},
  \label{eq:inference_target}
\end{equation}
where \(\bm{\mu}=(f_\star(\bm{x}_1),\ldots,f_\star(\bm{x}_N))^\top\).
For Example~1, the weight vector is \(\bm{\eta}^{\mathrm{reg}}_{\bm{Y}}:=|\mathcal{I}_{\bm{Y}}|^{-1}\bm{1}_{\mathcal{I}_{\bm{Y}}}-|\mathcal{J}_{\bm{Y}}|^{-1}\bm{1}_{\mathcal{J}_{\bm{Y}}}\),
where \(\bm{1}_{\mathcal{C}}\in\{0,1\}^N\) is the indicator vector of \(\mathcal{C}\subset[N]\).
For Example~2, the weight vector is \(
\bm\eta^{\mathrm{top}}_{\bm Y}
  :=
  n^{-1}\bm 1_{\mathcal I_{\bm Y}^{\mathrm{top},n}}
\).
Note that \(\bm\eta_{\bm Y}\) is a random variable that depends on the data \(\bm Y\).
Naturally, a test statistic for \eqref{eq:inference_target} is defined by
\begin{equation}
  T(\bm{Y})
  =
  \bm{\eta}_{\bm{Y}}^\top \bm{Y}.
  \label{eq:post_bo_random_stat}
\end{equation}

Examples~1 and~2 show typical data-dependent targets covered by our framework.
The common feature behind these examples is that the target is constructed after observing the data, through the map
\(\bm{Y}\mapsto\bm{\eta}_{\bm{Y}}\).
Our framework covers a broad class of such target-construction rules, formally defined in Appendix~\ref{app:admissibility_framework}.
Informally, the required condition is that the construction of \(\bm{\eta}_{\bm{Y}}\) can be described by a sequence of piecewise-linear operations.
This condition is satisfied, for example, when the weight vector is chosen through comparisons among the observed responses, as long as those comparisons can be characterized as linear inequalities in \(\bm{Y}\).
Thus, the framework includes a wide range of data-dependent targets, with data-independent targets appearing as a special case.
The two examples above satisfy this requirement.
Appendix~\ref{app:additional_targets} provides several additional examples.

One inference goal is to construct a valid confidence interval for \(\Delta_{\bm{Y}}\) using the test statistic \(T(\bm{Y})\).
A valid \((1-\alpha)\)-confidence interval \(C_{1-\alpha}^{\Delta}(\bm{Y})\) should satisfy
\begin{equation}
  \Prob\!\left(
    \Delta_{\bm{Y}} \in C_{1-\alpha}^{\Delta}(\bm{Y})
  \right)
  =
  1-\alpha.
  \label{eq:valid_ci}
\end{equation}
In Example~1, this corresponds to estimating the true performance difference between the high-performance and low-performance regions, while controlling the miscoverage probability at $\alpha$.

Another inference goal is to construct a valid $p$-value for testing the hypotheses:
\begin{equation}
  \mathrm{H}_{0,\bm{Y}}:\ \Delta_{\bm{Y}} \le 0
  \qquad
  \text{vs.}
  \qquad
  \mathrm{H}_{1,\bm{Y}}:\ \Delta_{\bm{Y}} > 0.
  \label{eq:post_bo_random_null}
\end{equation}
Formally, a valid $p$-value should satisfy
\begin{equation}
  \Prob_{\mathrm{H}_0}(p \le \alpha) \le \alpha
  \qquad
  \text{for all }
  \alpha\in(0,1).
  \label{eq:valid_pvalue}
\end{equation}
In Example~1, this tests whether the high-performance region significantly outperforms the low-performance region, rather than reflecting random fluctuation, while controlling the false finding probability at level $\alpha$.

Although we have focused here on a one-sided test and a two-sided confidence interval, our framework also covers the corresponding inverted procedures, namely the equal-tailed two-sided test and the lower one-sided confidence interval.
Appendix~\ref{app:formal_selective_inference} provides their formal definitions and properties.

\subsection{Why Naive Inference Fails}
\label{subsec:selection_bias}
The key challenge in this post-ADC analysis is that the same response vector \(\bm{Y}\) is used three times in this workflow.
First, \(\bm{Y}\) determines the ADC trajectory \(\mathcal T_{\bm Y}=(\bm x_1,\ldots,\bm x_N)\), i.e., the sequence of queried locations.
Thus, the ADC trajectory may tend to select locations with favorable outcomes for \(\bm Y\).
For example, in an ADC procedure such as SMBO for maximization, regions with high observed values tend to be explored more intensively, whereas regions that appear low due to random noise may be left underexplored.
Second, after this data-driven trajectory has been determined, the same \(\bm{Y}\) are used to define the regions
\(\mathcal{I}_{\bm{Y}}\) and \(\mathcal{J}_{\bm{Y}}\), and hence the weight vector
\(\bm{\eta}_{\bm{Y}}\), the target \(\Delta_{\bm{Y}}\), and the hypotheses in
\eqref{eq:post_bo_random_null}.
The important point is that \(\bm{\eta}_{\bm{Y}}\) is itself a function of \(\bm{Y}\).
For example, if the regions are chosen so that the observed difference along the ADC trajectory appears large, then the resulting
\(\bm{\eta}_{\bm{Y}}\) can reflect not only the underlying signal but also favorable random fluctuations in the observed responses.
Finally, the same \(\bm{Y}\) are used to compute the test statistic \(T(\bm{Y})\) and to perform inference.
Naive post-ADC inference ignores these multiple uses of data and treats the queried locations and the weight vector
\(\bm{\eta}_{\bm{Y}}\) as if they had been fixed in advance.
This is a statistically inappropriate procedure, and the validity guarantees in \eqref{eq:valid_pvalue} and \eqref{eq:valid_ci} are no longer guaranteed to hold.
As a result, naive post-ADC inference can exhibit inflated type-I error and confidence intervals whose coverage falls below the nominal level.
This is the same selection-bias phenomenon studied throughout the selective-inference literature
\citep{fithian2014optimal, lee2016exact}.
We will later confirm these failures empirically in \S~\ref{sec:experiments}.

\newpage
\section{Selective Inference for Post-ADC Analysis}
\label{sec:selective_pvalues}

As discussed in \S~\ref{subsec:selection_bias}, the main obstacle in post-ADC inference is that the locations of the queried points, the definition of the inferential target, and the test statistic all depend on the same data $\bm{Y}$.
We now propose a post-ADC inference framework to overcome this obstacle by leveraging SI principles.
The key idea is to condition on the selection event that determines the inferential target.

\subsection{Conditioning on the Selection Event}
\label{subsec:si_concept}
In the setup described in \S~\ref{sec:data_model} and \S~\ref{sec:target_of_inference}, the data are used to select both the queried trajectory \(\mathcal{T}_{\bm{Y}}\) and the data-dependent weight vector $\bm{\eta}_{\bm{Y}}$ that defines the inferential target.
To account for the bias induced by these data dependencies, we condition on the entire selection process \(\mathcal{E}_{\bm{Y}}
= (\mathcal{T}_{\bm{Y}},\bm{\eta}_{\bm{Y}})\) by its realization \(\mathcal{E}_{\bm{y}}\) and perform inference conditional on the selection event:
\begin{equation}
  \{\mathcal{E}_{\bm{Y}}=\mathcal{E}_{\bm{y}}\}
  =\{\mathcal{T}_{\bm{Y}}=\mathcal{T}_{\bm{y}},\ \bm{\eta}_{\bm{Y}}=\bm{\eta}_{\bm{y}}\}.
  \label{eq:selection_event}
\end{equation}
We then consider the conditional distribution of the test statistic given the selection event:
\begin{equation}
  T(\bm{Y})
  \ \big|\
  \{\mathcal{E}_{\bm{Y}}=\mathcal{E}_{\bm{y}}\}.
  \label{eq:conditional_test_statistic}
\end{equation}
This conditioning plays two key roles.
\paragraph{Conditioning on \(\mathcal{T}_{\bm{Y}}\).}
If the queried trajectory \(\mathcal{T}_{\bm{Y}}\) were fixed in advance, then \(\bm{Y}\) follows the fixed-mean Gaussian model:
\begin{equation}
  \bm{Y}\sim N(\bm{\mu},\Sigma),
  \qquad
  \bm{\mu}=(f_\star(\bm{x}_1),\ldots,f_\star(\bm{x}_N))^\top,
  \qquad
  \Sigma=\sigma^2 I_N.
  \label{eq:observed_model_gaussian}
\end{equation}
Conditioning on \(\mathcal{T}_{\bm{Y}}=\mathcal{T}_{\bm{y}}\) therefore restricts the above Gaussian distribution to the subset of \(\bm{Y}\) that reproduce the same queried trajectory \(\mathcal{T}_{\bm{y}}\).
Although the resulting conditional distribution is not an ordinary Gaussian distribution, it retains the Gaussian density structure, making it substantially more tractable than the unconditional distribution induced by the ADC process.

\paragraph{Conditioning on \(\bm{\eta}_{\bm{Y}}\).}
This fixes the target of inference \(\Delta_{\bm{Y}}=\bm{\eta}_{\bm{Y}}^\top\bm{\mu}\) and the corresponding hypotheses in \eqref{eq:post_bo_random_null}.
For the Example~1 in \S~\ref{subsec:target_of_inference}, fixing \(\bm{\eta}_{\bm{Y}}\) is equivalent to fixing the selected high and low regions.
Let \(\bm{\eta}=\bm{\eta}_{\bm{y}}\) and \(T_{\bm{\eta}}(\bm{Y})=\bm{\eta}^\top\bm{Y}\).
Then the inferential target in \eqref{eq:inference_target} is fixed at
\(
\Delta^{\mathrm{sel}}=\bm{\eta}^\top\bm{\mu}
\)
and the hypotheses in \eqref{eq:post_bo_random_null} become
\(\mathrm H_0:\Delta^{\mathrm{sel}}\le 0\) vs. \(\mathrm H_1:\Delta^{\mathrm{sel}}>0\).

Thus, after conditioning on the selection event, the statistic of interest \(\bm{\eta}^\top\bm{Y}\) becomes a linear function of \(\bm{Y}\), while the conditional distribution retains the Gaussian density structure restricted to the selected data space.
This allows us to characterize the conditional distribution in \eqref{eq:conditional_test_statistic} in a relatively tractable form in \S~\ref{sec:exact_computation}.

\subsection{Selective $p$-Values and Confidence Intervals}
\label{subsec:inferential_outputs}

The conditional distribution in \eqref{eq:conditional_test_statistic} still involves nuisance parameters in addition to the effect \(\Delta^{\mathrm{sel}}\).
Under \eqref{eq:observed_model_gaussian}, decompose the mean vector as
\begin{equation}
  \bm{\mu}
  =
  \underbrace{\Delta^{\mathrm{sel}}\,\frac{\Sigma\bm{\eta}}{\bm{\eta}^\top\Sigma\bm{\eta}}}_{\text{parameter of interest}}
  +
  \underbrace{
  \left(
    I_N
    -
    \frac{\Sigma\bm{\eta}\bm{\eta}^\top}
         {\bm{\eta}^\top \Sigma \bm{\eta}}
  \right)\bm{\mu}
  }_{\text{nuisance parameter (denoted as \(\bm{\mu}^{\mathrm{nui}}\))}}.
  \label{eq:observed_parameter_decomposition}
\end{equation}
Thus \(\Delta^{\mathrm{sel}}\) is the parameter of interest, while \(\bm{\mu}^{\mathrm{nui}}\) collects the nuisance component orthogonal to the selected weight vector.
To identify the conditional distribution in \eqref{eq:conditional_test_statistic}, we need to eliminate the nuisance parameter, which is difficult to marginalize out.
We therefore condition on its sufficient statistic, namely the residual after projection onto the weight-vector direction:
\begin{equation}
  \mathcal{Q}_{\bm{Y}}
  =
  \left(
    I_N
    -
    \frac{\Sigma\bm{\eta}\bm{\eta}^\top}
         {\bm{\eta}^\top \Sigma \bm{\eta}}
  \right)\bm{Y}.
  \label{eq:nuisance_parameter}
\end{equation}
A formal statement of this sufficient statistic and its proof are provided in Appendix~\ref{app:nuisance_reduction} and~\ref{app:proof_observable_nuisance_sufficiency} respectively.
Therefore, the conditional test statistic used for inference is
\begin{align}
  T_{\bm{\eta}}(\bm{Y}) \mid
  \left\{
    \mathcal{E}_{\bm{Y}} = \mathcal{E}_{\bm{y}},
    \mathcal{Q}_{\bm{Y}} = \mathcal{Q}_{\bm{y}}
  \right\}.
  \label{eq:conditional_test_statistic_2}
\end{align}

The corresponding cumulative distribution function (CDF) of the statistic is defined as
\begin{equation}
  G_{\Delta}^{\mathrm{sel}}(t)
  =
  \Prob_{\Delta^{\mathrm{sel}}=\Delta}\!\left(
    T_{\bm{\eta}}(\bm{Y})\le t
    \,\middle|\,
    \mathcal{E}_{\bm{Y}}=\mathcal{E}_{\bm{y}},
    \mathcal{Q}_{\bm{Y}}=\mathcal{Q}_{\bm{y}}
  \right).
  \label{eq:selective_cdf}
\end{equation}

The selective $p$-value for the directional claim in \eqref{eq:post_bo_random_null} is evaluated at the boundary value \(\Delta^{\mathrm{sel}}=0\):
\begin{equation}
  p_{\mathrm{sel}}(\bm{y})
  =
  \Prob_{\Delta^{\mathrm{sel}}=0}\!\left(
    T_{\bm{\eta}}(\bm{Y}) \ge T_{\bm{\eta}}(\bm{y})
    \,\middle|\,
    \mathcal{E}_{\bm{Y}}=\mathcal{E}_{\bm{y}},
    \mathcal{Q}_{\bm{Y}}=\mathcal{Q}_{\bm{y}}
  \right).
  \label{eq:selective_pvalue}
\end{equation}

The same conditional distribution also yields a selective confidence interval:
\begin{equation}
  C^{\Delta}_{1-\alpha}(\bm{y})
  =
  \bigl\{\Delta\in\mathbb R: \alpha/2 \le G_{\Delta}^{\mathrm{sel}}(T_{\bm{\eta}}(\bm{y})) \le 1-\alpha/2\bigr\}.
  \label{eq:selective_ci_delta}
\end{equation}

Appendix~\ref{app:formal_selective_inference} provides the corresponding complementary one-sided and two-sided selective outputs.
At this point the inferential outputs are fully specified by selective-inference principles, but their computation still depends on the conditional distribution in \eqref{eq:selective_cdf}.
The next section shows that the conditional distribution reduces to a truncated normal distribution that can be evaluated exactly.

\newpage
\section{Computing Selective $p$-Values and Confidence Intervals}
\label{sec:exact_computation}

In this section, we show that, for a broad class of ADC processes, the conditional distribution of the test statistic in \eqref{eq:conditional_test_statistic} admits a closed-form expression, and thus the selective $p$-values and confidence intervals can be computed exactly under the finite-sample setup.

\subsection{Conditional Distribution of the Test Statistic}
\label{subsec:conditional_distribution}

As discussed in \S~\ref{sec:selective_pvalues}, after conditioning on the selection event, the distribution for the data $\bm{Y}$ retains the Gaussian density structure in \eqref{eq:observed_model_gaussian}, and the data-dependent selected weight vector $\bm{\eta}_{\bm{Y}}$ is fixed at its realized value $\bm{\eta}$.
Consequently, the test statistic follows normal distribution, but truncated to the subset of $\bm{Y}$ that reproduce the same selection event as the observed data.
This observation is formalized in the following lemma.

\begin{lemma}
  \label{lem:line_reduction}
  Assume the Gaussian distribution in \eqref{eq:observed_model_gaussian}, and let \(\bm{y}\) be the observed realization of \(\bm{Y}\) with \(\bm{\eta}_{\bm{y}}\neq \bm{0}\).
  Fix \(\bm{\eta}=\bm{\eta}_{\bm{y}}\).
  Define
  \begin{equation}
    \bm{a}=\mathcal{Q}_{\bm{y}},
    \qquad
    \bm{b}=
    \frac{\Sigma \bm{\eta}}
         {\bm{\eta}^\top \Sigma \bm{\eta}}.
    \label{eq:line_parameters}
  \end{equation}
  Then every response vector on the resulting one-dimensional slice can be written as
  \begin{equation}
    \bm{Y}(z)=\bm{a}+\bm{b}z,
    \qquad
    z=T(\bm{Y}(z)).
    \label{eq:one_dimensional_slice}
  \end{equation}
  Let
  \begin{equation}
    \mathcal{Z}
    =
    \left\{
      z\in\R:
      \mathcal{E}_{\bm{Y}(z)}=\mathcal{E}_{\bm{y}}
    \right\}.
    \label{eq:truncation_set}
  \end{equation}
  Then the conditional distribution of $T(\bm{Y})$ given
  \(
  \mathcal{E}_{\bm{Y}}=\mathcal{E}_{\bm{y}}
  \)
  and
  \(
  \mathcal{Q}_{\bm{Y}}=\mathcal{Q}_{\bm{y}}
  \)
  is a truncated normal distribution:
  \[
  T(\bm{Y})\mid\{\mathcal{E}_{\bm{Y}}=\mathcal{E}_{\bm{y}},\mathcal{Q}_{\bm{Y}}=\mathcal{Q}_{\bm{y}}\}
   \sim \mathrm{TN}(\Delta^{\mathrm{sel}},v_{\bm{\eta}},\mathcal{Z})
  \]
  with mean $\Delta^{\mathrm{sel}}$ and variance $v_{\bm{\eta}}=\bm{\eta}^\top \Sigma \bm{\eta}$, truncated to $\mathcal{Z}$.
\end{lemma}

The proof of the lemma is given in Appendix~\ref{app:proof_line_reduction} and a toy example is provided in Appendix~\ref{subsec:toy_example} to give intuition behind the result.
The statement of Lemma~\ref{lem:line_reduction} indicates that the problem of characterizing the conditional distribution of the test statistic reduces to characterizing the truncation set \(\mathcal{Z}\) along the line \(\bm{Y}(z)=\bm{a}+\bm{b}z\).
Consequently, once \(\mathcal{Z}\) has been computed, the identified selective CDF \(G^{\mathrm{sel}}_\delta\) in \eqref{eq:selective_cdf} directly yields the selective $p$-value in \eqref{eq:selective_pvalue} and the selective confidence interval in \eqref{eq:selective_ci_delta}.
Furthermore, the selective $p$-value is valid, in the sense that it controls the type-I error under the null hypothesis \(\mathrm{H}_0\) in \eqref{eq:post_bo_random_null}.
Similarly, the confidence interval achieves the valid coverage for the selected effect \(\Delta^{\mathrm{sel}}\).
These results are formally stated in the following theorems.
\begin{theorem}
  \label{thm:validity}
  Under the setup of Lemma~\ref{lem:line_reduction}, suppose that \(\Prob(\bm{\eta}_{\bm{Y}}\neq \bm{0})=1\).
  The $p$-value defined in \eqref{eq:selective_pvalue} satisfies
  \begin{equation}
    \Prob_{\mathrm{H}_0}(p_{\mathrm{sel}}(\bm{Y}) \le \alpha) \le \alpha, \quad \forall \alpha \in (0,1).
    \label{eq:validity_theorem}
  \end{equation}
\end{theorem}
\begin{theorem}
  \label{thm:ci_validity}
  Under the assumptions of Theorem~\ref{thm:validity},
  the confidence interval defined in \eqref{eq:selective_ci_delta} satisfies
  \begin{equation}
    \Prob\!\left(
      \Delta_{\bm{Y}} \in C^{\Delta}_{1-\alpha}(\bm{Y})
    \right)
    =1-\alpha, \quad \forall \alpha \in (0,1).
    \label{eq:ci_marginal_validity}
  \end{equation}
\end{theorem}
Their proofs are given in Appendix~\ref{app:proof_validity}.
Theorems~\ref{thm:validity} and \ref{thm:ci_validity} show that the proposed post-ADC inference enables us to draw insights from ADC data while controlling the probability of false findings at a pre-specified level \(\alpha\).
Proposition~\ref{prop:appendix_one_sided_validity} in Appendix~\ref{app:one_sided_test} further proves the stronger statement that, under the simple null hypothesis \(\Delta^{\mathrm{sel}}=0\), the selective $p$-value is exactly uniform.
Appendix~\ref{app:formal_selective_inference} also provides the formal definition and properties of the two-sided selective test and the one-sided lower confidence interval, which are also valid by the same argument as above.

\subsection{Computation of the Truncation Set \(\mathcal{Z}\)}
\label{subsec:truncation_set}

To obtain an exact computational procedure from the conditional distribution above, we must determine the set \(\mathcal{Z}\) of statistic values \(z\) that reproduce the same observed selection event as the observed data.
Due to space limitations, we omit the details of how to compute $\mathcal{Z}$; see Appendix~\ref{app:computation} and the accompanying code for the full procedure.
Roughly speaking, the piecewise-linearity of ADC algorithms and target-construction rules introduced in \S~\ref{sec:problem_setup} and \S~\ref{sec:target_of_inference} and formalized in Appendix~\ref{app:admissibility_framework} imply that the selection event can be represented in the response space by finitely many linear equalities and inequalities in \(\bm{Y}\).
Restricting those constraints to the one-dimensional slice \(\bm{Y}(z)=\bm{a}+\bm{b}z\) from Lemma~\ref{lem:line_reduction} turns the problem of computing \(\mathcal{Z}\) into solving finitely many linear inequalities in \(z\).
Appendix~\ref{app:interval_computation} provides the explicit derivations for GP-UCB and TPE.
Existing SI literature has pointed out the usefulness of randomization for stabilizing inferential outputs~\citep{tian2018selective, kivaranovic2024tight}.
Our framework can also be extended to incorporate randomization while retaining the validity guarantees.
This extension is described in Appendix~\ref{app:randomized_extension}.

\newpage
\section{Experiments}
\label{sec:experiments}

We evaluate post-ADC inference on synthetic and real data, assessing type-I error, coverage, and power.
Unless stated otherwise, we use dimension $d=3$, Gaussian noise variance $\sigma^2=1$, $N_{\mathrm{init}}=10$, and $N_{\mathrm{steps}}=50$, giving $N=N_{\mathrm{init}}+N_{\mathrm{steps}}$ samples.
We compare TPE \eqref{eq:tpe_score} and GP-UCB \eqref{eq:GPUCB}, using the high-region-versus-low-region target from \S~\ref{sec:target_of_inference}.
Tests use $\alpha=0.05$, intervals use $0.90$ coverage, and rates are estimated from $5{,}000$ independent Monte Carlo replicates per configuration.
Implementation details are in Appendix~\ref{app:appendix_setup_details}.
We compare the proposed method (\proposed) against four baselines: \naive, a $Z$-test without selection correction; \withouteta and \withoutT are ablations, which remove conditioning on $\bm{\eta}_{\bm{Y}}$ and $\mathcal{T}_{\bm{Y}}$, respectively; and \bonferroni, which corrects for $|\mathcal{X}|^{N_{\mathrm{steps}}}\cdot 3^{|\mathcal{X}|}$ possible trajectories and region selections.

\subsection{Type-I Error and Coverage}
\label{subsec:experiments_fpr}
\label{subsec:experiments_coverage}

Under the global null $f_\star\equiv 0$, we vary $N_{\mathrm{steps}}\in\{25,50,75,100\}$ and report rejection probability at $\alpha=0.05$ and coverage of nominal $90\%$ intervals in Figure~\ref{fig:validity}.
\proposed stays near the nominal levels, while \naive and the partially conditioned ablations are anti-conservative and \bonferroni is overly conservative, matching Theorem~\ref{thm:validity} and Theorem~\ref{thm:ci_validity}.
Additional dimensionality and hyperparameter sensitivity analyses are provided in Appendices~\ref{app:appendix_fpr_dimension} and~\ref{app:appendix_fpr_hyperparameter}, respectively, while interval-width results are reported in Appendix~\ref{app:appendix_ci_length}.

\begin{figure}[t]
  \centering
  \setlength{\tabcolsep}{0pt}
  \begin{tabular}{@{}c@{\hspace{0.02\textwidth}}c@{}}
    \multicolumn{2}{@{}c@{}}{
      \includegraphics[page=3,height=2.2cm,trim=0.2cm 1.5cm 2.0cm 2.0cm,clip]{experiments/coverage_nsteps.pdf}
    }\\[-0.25em]
    \includegraphics[page=1,width=0.49\textwidth,trim=0.0cm 0.0cm 1.0cm 0.90cm,clip]{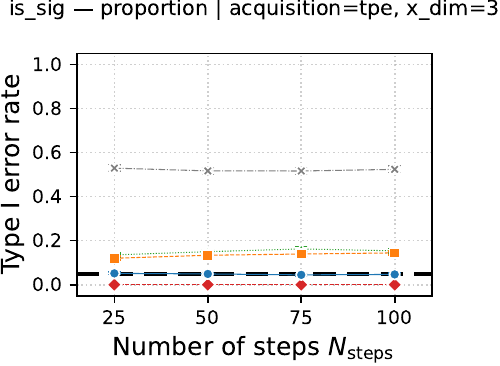} &
    \includegraphics[page=2,width=0.49\textwidth,trim=0.0cm 0.0cm 1.0cm 0.90cm,clip]{experiments/fpr_nsteps.pdf} \\[-0.55em]
    {Type I Error (\texttt{TPE})} &
    {Type I Error (\texttt{GP-UCB})} \\[2.0em]
    \includegraphics[page=1,width=0.49\textwidth,trim=0.2cm 0.0cm 1.4cm 0.90cm,clip]{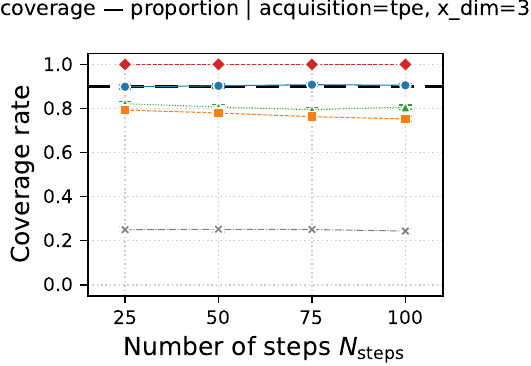} &
    \includegraphics[page=2,width=0.49\textwidth,trim=0.2cm 0.0cm 1.4cm 0.90cm,clip]{experiments/coverage_nsteps.pdf} \\[-0.55em]
    {Coverage (\texttt{TPE})} &
    {Coverage (\texttt{GP-UCB})}\\
  \end{tabular}
  \caption{Empirical Type-I error and coverage.
    \proposed remains close to nominal type-I error and coverage across horizons, while \naive and the ablations are anti-conservative and \bonferroni is conservative.  In GP-UCB, \withouteta nearly overlaps \proposed; failure cases appear in Appendix~\ref{app:appendix_fpr_hyperparameter}.}
  \label{fig:validity}
  \label{fig:fpr}
  \label{fig:coverage}
\end{figure}

\subsection{Power Experiments}
\label{subsec:experiments_power}
We test six synthetic objective families (\texttt{sinc}, \texttt{cos}, \texttt{chirp}, \texttt{bump}, \texttt{peak}, and \texttt{negative\_forrester}).
They are rescaled to $[-1,1]$ and multiplied by $a\in\{0,1,2,4,8\}$; functional forms are in Appendix~\ref{app:synthetic_objective_families}.
Because \naive, \withoutT, and \withouteta are anti-conservative under the null in Figure~\ref{fig:validity}, Figure~\ref{fig:power} compares only \proposed and \bonferroni.
\proposed is calibrated at $a=0$ and gains power quickly as $a$ increases.

\begin{figure}[t]
  \centering
  \setlength{\tabcolsep}{0pt}
  \begin{tabular}{@{}c@{\hspace{0.02\textwidth}}c@{}}
    \multicolumn{2}{@{}c@{}}{
      \includegraphics[page=3,width=0.72\textwidth,trim=0.0cm 3.5cm 0.05cm 4.0cm,clip]{experiments/power_signal.pdf}
    }\\[-0.4em]
    \includegraphics[page=1,width=0.49\textwidth,trim=0.0cm 0.0cm 0.0cm 0.90cm,clip]{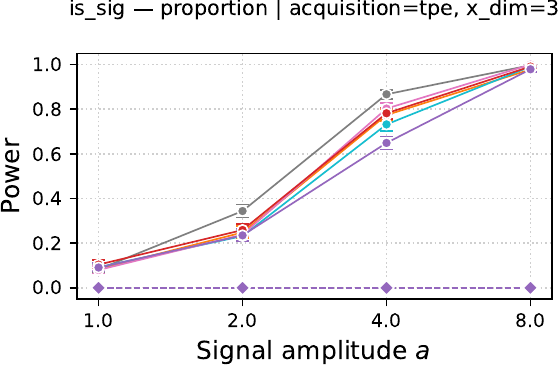} &
    \includegraphics[page=2,width=0.49\textwidth,trim=0.0cm 0.0cm 0.0cm 0.90cm,clip]{experiments/power_signal.pdf} \\[-0.5em]
    {\scriptsize \texttt{TPE}} &
    {\scriptsize \texttt{GP-UCB}}\\
  \end{tabular}
  \caption{Empirical power.
    \proposed gains power as the signal increases across six synthetic objective families, while all \bonferroni results (dashed lines) remain near zero.}
  \label{fig:power}
\end{figure}

\subsection{Experiments on Real Data}
\label{subsec:experiments_real}

We use four UCI benchmark datasets as black-box objectives, with dataset, feature-selection, and noise-estimation details in Appendix~\ref{app:appendix_real_data}.
For each configuration, we run $1{,}000$ bootstrap replicates.
Table~\ref{tab:real_d1} gives the $d=1$ results; \proposed detects signal on all datasets, whereas \bonferroni is essentially powerless.
Analogous results for varying $d$ are in Appendix~\ref{app:appendix_real_data}.

\begin{table}[t]
  \centering
  \caption{Empirical rejection probability on four real datasets.}
  \label{tab:real_d1}
  \small
  \setlength{\tabcolsep}{3pt}
  \begin{tabular}{@{}llcccc@{}}
  \toprule
  Acquisition & Method & Concrete & \makecell{Gas Turbine\\(CO)} & \makecell{Gas Turbine\\(NOx)} & \makecell{Power\\Plant} \\
  \midrule
  \multirow{2}{*}{ucb} & \proposed & \textbf{0.432} & \textbf{0.995} & \textbf{0.911} & \textbf{0.999} \\
  & \bonferroni & 0.000 & 0.014 & 0.000 & 0.000 \\
  \midrule
  \multirow{2}{*}{tpe} & \proposed & \textbf{0.396} & \textbf{0.989} & \textbf{0.918} & \textbf{0.929} \\
  & \bonferroni & 0.000 & 0.058 & 0.000 & 0.000 \\
  \bottomrule
  \end{tabular}
\end{table}
\newpage
\section{Discussion and Limitations}
\label{sec:discussion}

We introduced post-ADC inference for valid inference after active data collection.
The key issue is that the same noisy observations determine both where ADC samples and which statistical question is later tested.
By treating this pipeline as a single selection event, we derive exact selective $p$-values and confidence intervals under the known-variance Gaussian model.
The framework therefore allows us to make post hoc statements from ADC data while controlling the probability of false findings.
Several scope restrictions remain important.
First, the exact theory relies on finite candidate sets, so continuous domains require discretization.
Second, the analysis assumes known Gaussian noise.
Extending the framework to estimated variance or non-Gaussian response models remains open.
Third, exact computation requires that the ADC rule and the target-construction rule can be expressed as compositions of finitely many piecewise-linear maps.
This covers TPE and GP-UCB, but not every ADC algorithm.
These limitations are important targets for future work.

\newpage
\subsubsection*{Acknowledgments}
This work was supported by RIKEN Junior Research Associate Program and partially supported by JST CREST (JPMJCR21D3, JPMJCR22N2),
JST Moonshot R\&D (JPMJMS2033-05), and RIKEN Center for Advanced Intelligence\\
Project.

\clearpage
\appendix

\newpage
\section{Additional Targets of Inference}

\label{app:additional_targets}

The main text already discussed two data-dependent target constructions in \S~\ref{subsec:target_of_inference}.
Our framework applies more widely to any target construction rule, which is formally defined in Appendix~\ref{app:admissibility_framework}.
This subsection collects additional examples of targets that fit our framework, including both data-dependent and data-independent targets.

\paragraph{Data-dependent targets.}
We summarize two additional data-dependent weight vector families in Table~\ref{tab:appendix_hypotheses}.
They all fit the same linear form
\[
  \Delta_{\bm{y}} = \bm{\eta}_{\bm{y}}^\top \bm{\mu}.
\]

\begin{table}[htbp]
  \centering
  \caption{Additional data-dependent target families beyond the two main-text examples.}
  \label{tab:appendix_hypotheses}
  \begingroup
  \small
  \setlength{\tabcolsep}{3pt}
  \renewcommand{\arraystretch}{1.2}
  \begin{tabular}{@{}p{0.05\linewidth}p{0.18\linewidth}p{0.38\linewidth}p{0.31\linewidth}@{}}
    \toprule
    \# & Name & $\mathcal{I}_{\bm{y}}$ & $\mathcal{J}_{\bm{y}}$ \\
    \midrule
    1 & Top-$n$ vs.
bottom-$m$ &
    \makecell[l]{$\displaystyle \mathcal{I}_{\bm{y}}=\operatorname*{arg\,max}_{|\mathcal{S}|=n}\sum_{t\in\mathcal{S}} y_t$} &
    \makecell[l]{$\displaystyle \mathcal{J}_{\bm{y}}=\operatorname*{arg\,min}_{|\mathcal{S}|=m}\sum_{t\in\mathcal{S}} y_t$} \\
    2 & Winner vs.
runner-up &
    \makecell[l]{$\displaystyle \mathcal{I}_{\bm{y}}=\{\hat{i}\},$ \\
                  $\displaystyle \hat{i}=\arg\max_{t\in[N]} y_t$} &
    \makecell[l]{$\displaystyle \mathcal{J}_{\bm{y}}=\{\tilde{j}\},$ \\
                  $\displaystyle \tilde{j}=\arg\max_{t\in[N]\setminus\{\hat{i}\}} y_t$} \\
    \bottomrule
  \end{tabular}
  \endgroup
\end{table}

\paragraph{Data-independent targets.}
The framework also covers useful data-independent targets, where the weight vector is fixed once the queried design points are fixed.
A first example is the observed GP posterior mean at a target location $\bm{x}^{\circ}$,
\begin{equation}
  \begin{aligned}
  \widehat{\mu}_N(\bm{x}^{\circ})
  &=
  \bm{k}_N(\bm{x}^{\circ})^\top
  (K_N+\sigma^2 I_N)^{-1}
  \bm{y}
  =
  (\bm{\eta}^{\mathrm{GP}}_{\mathcal{T}})^\top \bm{y},
  \\
  \bm{\eta}^{\mathrm{GP}}_{\mathcal{T}}
  &:=
  (K_N+\sigma^2 I_N)^{-1}\bm{k}_N(\bm{x}^{\circ}).
  \end{aligned}
  \label{eq:appendix_gp_posterior_mean_weight vector}
\end{equation}

A second example is the observed average response over a pre-specified region $\mathcal{R}\subset\R^d$,
\begin{equation}
  \overline{y}(\mathcal{R})
  =
  \frac{1}{|\mathcal{I}(\mathcal{R})|}
  \sum_{t\in\mathcal{I}(\mathcal{R})} y_t
  =
  (\bm{\eta}^{\mathcal{R}}_{\mathcal{T}})^\top \bm{y},
  \qquad
  \mathcal{I}(\mathcal{R}) := \{t:\ \bm{x}_t\in\mathcal{R}\},
  \label{eq:appendix_fixed_region_average}
\end{equation}
provided $|\mathcal{I}(\mathcal{R})|>0$.

In these data-independent cases, conditioning remains essential because the ADC trajectory depends on the responses, but the selection event can be simplified to \(\{\mathcal{T}_{\bm{Y}}=\mathcal{T}_{\bm{y}}\}\).
Our framework still yields valid selective tests and confidence intervals for these simpler targets.

\newpage
\section{Formal Selective Inference Statements}
\label{app:formal_selective_inference}

This appendix provides the formal statements about the elimination of the nuisance parameter by conditioning on the residual statistic in \eqref{eq:nuisance_parameter}.
It also provides the formal definitions and properties of the selective $p$-values and confidence intervals.
The one-sided selective $p$-value in \eqref{eq:selective_pvalue} and the equal-tailed two-sided selective confidence interval in \eqref{eq:selective_ci_delta} have already been formally defined in \S~\ref{subsec:inferential_outputs}.
Therefore, this appendix focuses on the complementary one-sided selective confidence interval and the two-sided selective test, providing their formal definitions and basic properties.

Throughout, let \(G^{\mathrm{sel}}_\delta(\cdot)\) denote the conditional CDF in \eqref{eq:selective_cdf} evaluated at effect value \(\delta\).

\subsection{Elimination of Nuisance Parameter}
\label{app:nuisance_reduction}

In \S~\ref{subsec:inferential_outputs}, we eliminate the nuisance parameter from the conditional distribution by conditioning on an appropriate residual statistic.
We formalize that step here.

\begin{proposition}[Nuisance parameter-elimination by conditioning]
  \label{prop:appendix_observable_nuisance_sufficiency}
  Let \(\bm{\eta}=\bm{\eta}_{\bm{y}}\), and assume the Gaussian model in \eqref{eq:observed_model_gaussian}.
  Define
  \begin{equation}
    \mathcal{Q}_{\bm{Y}}
    =
    \left(
      I_N
      -
      \frac{\Sigma\bm{\eta}\bm{\eta}^\top}
           {\bm{\eta}^\top \Sigma \bm{\eta}}
    \right)\bm{Y}.
    \label{eq:appendix_nuisance_parameter}
  \end{equation}
  Then, for any fixed \(q\), the conditional distribution of
  \(T_{\bm{\eta}}(\bm{Y})=\bm{\eta}^{\top}\bm{Y}\) given
  \[
    \mathcal{E}_{\bm{Y}}=\mathcal{E}_{\bm{y}},
    \qquad
    \mathcal{Q}_{\bm{Y}}=q
  \]
  depends on the mean parameter \(\bm{\mu}\) only through
  \(\Delta^{\mathrm{sel}}=\bm{\eta}^{\top}\bm{\mu}\).
  In particular, it does not depend on the nuisance component
  \(\bm{\mu}^{\mathrm{nui}}\) in \eqref{eq:observed_parameter_decomposition}.
\end{proposition}

The proof is given in Appendix~\ref{app:proof_observable_nuisance_sufficiency}.
Proposition~\ref{prop:appendix_observable_nuisance_sufficiency} shows that the residual obtained after projecting onto the selected weight-vector direction can be conditioned on to remove the nuisance component from the conditional law of the target statistic.
Conditioning on \(\mathcal{Q}_{\bm{Y}}=\mathcal{Q}_{\bm{y}}\) therefore yields a selective distribution for \(T_{\bm{\eta}}(\bm{Y})\) that depends on the mean parameter only through \(\Delta^{\mathrm{sel}}\).
In this sense, \(\mathcal{Q}_{\bm{Y}}\) plays the role of a sufficient statistic for the nuisance parameter.
The elimination of nuisance parameters by conditioning is standard in conditional inference \citep{lehmann1986testing}, and the same residual statistic appears in classical selective inference for linear regression \citep{lee2016exact, tibshirani2016exact}.

\subsection{One-Sided Selective Test}
\label{app:one_sided_test}

The one-sided selective $p$-value used in the main text is already defined formally in \eqref{eq:selective_pvalue}.
It tests
\(
\mathrm{H}_0:\Delta^{\mathrm{sel}}\le 0
\)
versus
\(
\mathrm{H}_1:\Delta^{\mathrm{sel}}>0
\).
Its validity is established in Theorem~\ref{thm:validity}; we also provide the stronger exact-uniformity statement.

\begin{proposition}[Validity of the one-sided selective test]
  \label{prop:appendix_one_sided_validity}
  Under the assumptions of Theorem~\ref{thm:validity}, the test that rejects for small values of the induced $p$-value \(p_{\mathrm{sel}}(\bm{Y})\) is exactly uniform under the simple null hypothesis and valid for the composite null.
In particular, for every \(\alpha\in(0,1)\),
  \begin{equation}
    \Prob_{\Delta^{\mathrm{sel}}=0}\!\left(
      p_{\mathrm{sel}}(\bm{Y})\le \alpha
      \,\middle|\,
      \mathcal{E}_{\bm{Y}},\mathcal{Q}_{\bm{Y}}
    \right)
    =
    \alpha
    \qquad \text{a.s.,}
    \label{eq:appendix_one_sided_conditional_exact}
  \end{equation}
  and hence
  \begin{equation}
    \Prob_{\Delta^{\mathrm{sel}}=0}(p_{\mathrm{sel}}(\bm{Y})\le \alpha)
    =
    \alpha.
    \label{eq:appendix_one_sided_marginal_exact}
  \end{equation}
  Consequently,
  \begin{equation}
    \Prob_{\mathrm{H}_0}\!\left(
      p_{\mathrm{sel}}(\bm{Y})\le \alpha
      \,\middle|\,
      \mathcal{E}_{\bm{Y}},\mathcal{Q}_{\bm{Y}}
    \right)
    \le
    \alpha
    \qquad \text{a.s.,}
    \label{eq:appendix_one_sided_conditional}
  \end{equation}
  and hence
  \begin{equation}
    \Prob_{\mathrm{H}_0}(p_{\mathrm{sel}}(\bm{Y})\le \alpha)\le \alpha.
    \label{eq:appendix_one_sided_marginal}
  \end{equation}
\end{proposition}

\subsection{One-Sided Selective Confidence Interval}
\label{app:one_sided_ci}

As noted at the end of \S~\ref{subsec:si_concept}, the same conditional distribution also supports the inversion of the directional selective test to a lower one-sided selective confidence interval.

\begin{definition}[Lower one-sided selective confidence interval]
  \label{def:appendix_one_sided_ci}
  The lower one-sided selective confidence interval for the effect is obtained by inverting the family of one-sided tests
  \(
  \mathrm{H}_0:\Delta^{\mathrm{sel}}\le \delta
  \)
  versus
  \(
  \mathrm{H}_1:\Delta^{\mathrm{sel}}>\delta
  \),
  namely
  \begin{equation}
    C^{\Delta,+}_{1-\alpha}(\bm{y})
    =
    \left\{
      \delta\in\R:
      G^{\mathrm{sel}}_\delta\!\left(T(\bm y)\right)
      \le
      1-\alpha
    \right\}.
    \label{eq:appendix_one_sided_ci}
  \end{equation}
\end{definition}

\begin{proposition}[Properties of the lower one-sided selective confidence interval]
  \label{prop:appendix_one_sided_ci}
  Under the assumptions of Theorem~\ref{thm:validity}, the set \(C^{\Delta,+}_{1-\alpha}(\bm{y})\) is a lower interval for each observed response vector \(\bm{y}\), and the induced interval \(C^{\Delta,+}_{1-\alpha}(\bm{Y})\) has exact coverage both conditionally and marginally.
Writing
  \(
  \Delta_{\bm{Y}}=\bm{\eta}_{\bm{Y}}^\top \bm{\mu}
  \)
  for the effect, we have
  \begin{equation}
    \Prob\!\left(
      \Delta_{\bm{Y}}\in C^{\Delta,+}_{1-\alpha}(\bm{Y})
      \,\middle|\,
      \mathcal{E}_{\bm{Y}},\mathcal{Q}_{\bm{Y}}
    \right)
    =
    1-\alpha
    \qquad \text{a.s.,}
    \label{eq:appendix_one_sided_ci_conditional}
  \end{equation}
  and therefore
  \begin{equation}
    \Prob\!\left(
      \Delta_{\bm{Y}}\in C^{\Delta,+}_{1-\alpha}(\bm{Y})
    \right)
    =
    1-\alpha.
    \label{eq:appendix_one_sided_ci_marginal}
  \end{equation}
\end{proposition}

\subsection{Two-Sided Selective Test}
\label{app:two_sided_test}

The same conditional distribution also supports an equal-tailed two-sided selective test of the point null.

\begin{definition}[Equal-tailed two-sided selective test]
  \label{def:appendix_two_sided_test}
  The equal-tailed two-sided selective $p$-value for testing
  \(
  \mathrm{H}_0:\Delta^{\mathrm{sel}}=0
  \)
  versus
  \(
  \mathrm{H}_1:\Delta^{\mathrm{sel}}\neq 0
  \)
  is
  \begin{equation}
    p_{\mathrm{sel}}^{\mathrm{ts}}(\bm{y})
    =
    2\min\!\left\{
      G^{\mathrm{sel}}_0\!\left(T(\bm{y})\right),
      1-G^{\mathrm{sel}}_0\!\left(T(\bm{y})\right)
    \right\}.
    \label{eq:appendix_two_sided_pvalue}
  \end{equation}
\end{definition}

\begin{proposition}[Validity of the equal-tailed two-sided selective test]
  \label{prop:appendix_two_sided_test}
  Under the assumptions of Theorem~\ref{thm:validity}, the test that rejects for small values of \(p_{\mathrm{sel}}^{\mathrm{ts}}(\bm{Y})\) is exactly uniform under the point null, both conditionally and marginally.
In particular, for every \(\alpha\in(0,1)\),
  \begin{equation}
    \Prob_{\Delta^{\mathrm{sel}}=0}\!\left(
      p_{\mathrm{sel}}^{\mathrm{ts}}(\bm{Y})\le \alpha
      \,\middle|\,
      \mathcal{E}_{\bm{Y}},\mathcal{Q}_{\bm{Y}}
    \right)
    =
    \alpha
    \qquad \text{a.s.,}
    \label{eq:appendix_two_sided_conditional_exact}
  \end{equation}
  and hence
  \begin{equation}
    \Prob_{\Delta^{\mathrm{sel}}=0}\!\left(
      p_{\mathrm{sel}}^{\mathrm{ts}}(\bm{Y})\le \alpha
    \right)
    =
    \alpha.
    \label{eq:appendix_two_sided_marginal_exact}
  \end{equation}
\end{proposition}

\subsection{Two-Sided Selective Confidence Interval}
\label{app:two_sided_ci}

The equal-tailed two-sided selective confidence interval used in the main text is already defined formally in \eqref{eq:selective_ci_delta}.
Equivalently, it can be written as the inversion of the equal-tailed two-sided selective tests:
\begin{equation}
  C^{\Delta}_{1-\alpha}(\bm{y})
  =
  \left\{
    \delta\in\R:
    2\min\!\left\{
      G^{\mathrm{sel}}_\delta\!\left(T(\bm y)\right),
      1-G^{\mathrm{sel}}_\delta\!\left(T(\bm y)\right)
    \right\}
    \ge
    \alpha
  \right\}.
  \label{eq:appendix_two_sided_ci_inversion}
\end{equation}

\begin{proposition}[Properties of the two-sided selective confidence interval]
  \label{prop:appendix_two_sided_ci}
  Under the assumptions of Theorem~\ref{thm:validity}, the set \(C^{\Delta}_{1-\alpha}(\bm{y})\) is an interval for each observed response vector \(\bm{y}\), and the induced interval \(C^{\Delta}_{1-\alpha}(\bm{Y})\) has exact coverage both conditionally and marginally.
Writing
  \(
  \Delta_{\bm{Y}}=\bm{\eta}_{\bm{Y}}^\top \bm{\mu}
  \)
  for the effect, we have
  \begin{equation}
    \Prob\!\left(
      \Delta_{\bm{Y}}\in C^{\Delta}_{1-\alpha}(\bm{Y})
      \,\middle|\,
      \mathcal{E}_{\bm{Y}},\mathcal{Q}_{\bm{Y}}
    \right)
    =
    1-\alpha
    \qquad \text{a.s.,}
    \label{eq:appendix_two_sided_conditional_delta}
  \end{equation}
  and therefore
  \begin{equation}
    \Prob\!\left(
      \Delta_{\bm{Y}}\in C^{\Delta}_{1-\alpha}(\bm{Y})
    \right)
    =
    1-\alpha.
    \label{eq:appendix_two_sided_marginal}
  \end{equation}
\end{proposition}

\newpage
\section{Proofs}
\label{app:proofs}

\subsection{Proof of Proposition~\ref{prop:appendix_observable_nuisance_sufficiency}}
\label{app:proof_observable_nuisance_sufficiency}

\begin{proof}
Under the Gaussian model in \eqref{eq:observed_model_gaussian}, let
\[
  T_{\bm{\eta}}(\bm{Y})=\bm{\eta}^{\top}\bm{Y},
  \qquad
  v_{\bm{\eta}}=\bm{\eta}^{\top}\Sigma\bm{\eta},
\]
and define
\[
  \mathcal{Q}_{\bm{Y}}
  =
  \left(
    I_N-\frac{\Sigma\bm{\eta}\bm{\eta}^{\top}}{v_{\bm{\eta}}}
  \right)\bm{Y}.
\]
Then
\[
  \operatorname{Cov}\!\left(
    \mathcal{Q}_{\bm{Y}},T_{\bm{\eta}}(\bm{Y})
  \right)
  =
  \left(
    I_N-\frac{\Sigma\bm{\eta}\bm{\eta}^{\top}}{v_{\bm{\eta}}}
  \right)\Sigma\bm{\eta}
  =
  0.
\]
Since \((T_{\bm{\eta}}(\bm{Y}),\mathcal{Q}_{\bm{Y}})\) is jointly Gaussian, it follows that
\(T_{\bm{\eta}}(\bm{Y})\) and \(\mathcal{Q}_{\bm{Y}}\) are independent.
Moreover,
\[
  T_{\bm{\eta}}(\bm{Y})\sim
  N\!\left(\Delta^{\mathrm{sel}},v_{\bm{\eta}}\right),
  \qquad
  \Delta^{\mathrm{sel}}=\bm{\eta}^{\top}\bm{\mu}.
\]
For any fixed \(q\), every vector \(\bm{z}\) satisfying
\(\mathcal{Q}_{\bm{z}}=q\) can be written as
\[
  \bm{z}
  =
  \frac{\Sigma\bm{\eta}}{v_{\bm{\eta}}}t+q,
  \qquad
  t=T_{\bm{\eta}}(\bm{z}).
\]
Therefore, after conditioning on both
\(\mathcal{Q}_{\bm{Y}}=q\) and the selection event
\(\mathcal{E}_{\bm{Y}}=\mathcal{E}_{\bm{y}}\), the conditional density of
\(T_{\bm{\eta}}(\bm{Y})\) is proportional to
\[
  \mathbf 1\!\left\{
    \mathcal{E}_{\Sigma\bm{\eta}t/v_{\bm{\eta}}+q}
    =
    \mathcal{E}_{\bm{y}}
  \right\}
  \exp\!\left\{
    -\frac{(t-\Delta^{\mathrm{sel}})^2}{2v_{\bm{\eta}}}
  \right\}.
\]
The normalizing constant depends on the truncation region induced by the
selection event, but it depends on the mean parameter only through
\(\Delta^{\mathrm{sel}}\). Hence the nuisance parameter is eliminated by
conditioning on \(\mathcal{Q}_{\bm{Y}}\).
\end{proof}

\subsection{Proof of Lemma~\ref{lem:line_reduction}}
\label{app:proof_line_reduction}

\begin{proof}
Fix the observed selection event, and let \(\bm{\eta}=\bm{\eta}_{\bm{y}}\) so that \(\mathcal{E}_{\bm{Y}}=\mathcal{E}_{\bm{y}}\) implies \(\bm{\eta}_{\bm{Y}}=\bm{\eta}\).
Let \(f^{\mathrm{pre}}\) denote densities under the pre-selection Gaussian distribution from \S~\ref{subsec:inferential_outputs}.

For notational simplicity, let \(v_{\bm{\eta}}=\bm{\eta}^\top\Sigma\bm{\eta}\), and define
\[
  T_{\bm{\eta}}(\bm{Y}) := \bm{\eta}^\top \bm{Y}.
\]
Under the conditioning on \(\mathcal{Q}_{\bm{Y}}=\mathcal{Q}_{\bm{y}}\), we have
\[
  \mathcal{Q}_{\bm{Y}}=\mathcal{Q}_{\bm{y}}
  \iff
  \left(I_N-\frac{\Sigma\bm{\eta}\bm{\eta}^\top}{v_{\bm{\eta}}}\right)\bm{Y}
  =\mathcal{Q}_{\bm{y}}
  \iff
  \bm{Y}=\bm{a}+\bm{b}z
\]
for \(z=\bm{\eta}^\top\bm{Y}\in\mathbb{R}\).

Hence,
\[
  \begin{aligned}
    \left\{
      \bm{Y}\in\mathbb{R}^N :
      \mathcal{E}_{\bm{Y}}=\mathcal{E}_{\bm{y}},
      \mathcal{Q}_{\bm{Y}}=\mathcal{Q}_{\bm{y}}
    \right\}
    &=
    \left\{
      \bm{Y}\in\mathbb{R}^N :
      \mathcal{E}_{\bm{Y}}=\mathcal{E}_{\bm{y}},
      \bm{Y}=\bm{a}+\bm{b}z,
      z\in\mathbb{R}
    \right\}\\
    &=
    \left\{
      \bm{Y}=\bm{a}+\bm{b}z\in\mathbb{R}^N :
      \mathcal{E}_{\bm{a}+\bm{b}z}=\mathcal{E}_{\bm{y}},
      \ z\in\mathbb{R}
    \right\}\\
    &=
    \left\{
      \bm{Y}=\bm{a}+\bm{b}z\in\mathbb{R}^N :
      z\in\mathcal{Z}
    \right\},
  \end{aligned}
\]
where
\[
  \mathcal{Z}
  :=
  \left\{
    z\in\mathbb{R} :
    \mathcal{E}_{\bm{a}+\bm{b}z}=\mathcal{E}_{\bm{y}}
  \right\}.
\]
This proves \eqref{eq:one_dimensional_slice}.

Under the pre-selection Gaussian distribution, \(T_{\bm{\eta}}(\bm{Y})\) and \(\mathcal{Q}_{\bm{Y}}\) are jointly Gaussian. Their covariance is
\[
\operatorname{Cov}\!\left(T_{\bm{\eta}}(\bm{Y}),\mathcal{Q}_{\bm{Y}}\right)
=
\bm{\eta}^\top\Sigma
\left(I_N-\frac{\Sigma\bm{\eta}\bm{\eta}^\top}{v_{\bm{\eta}}}\right)^\top
=
\bm{0}^\top.
\]
Therefore, \(T_{\bm{\eta}}(\bm{Y})\) and \(\mathcal{Q}_{\bm{Y}}\) are independent under the pre-selection Gaussian distribution, and hence
\[
f^{\mathrm{pre}}_{\,T_{\bm{\eta}}(\bm{Y}),\,\mathcal{Q}_{\bm{Y}}}(t,q)
=
f^{\mathrm{pre}}_{\,T_{\bm{\eta}}(\bm{Y})}(t)\,
f^{\mathrm{pre}}_{\,\mathcal{Q}_{\bm{Y}}}(q),
\]
and therefore
\[
f^{\mathrm{pre}}_{\,T_{\bm{\eta}}(\bm{Y}) \mid \mathcal{Q}_{\bm{Y}}=\mathcal{Q}_{\bm{y}}}(t)
=
f^{\mathrm{pre}}_{\,T_{\bm{\eta}}(\bm{Y})}(t).
\]

Also,
\[
T_{\bm{\eta}}(\bm{Y}) \sim \mathcal{N}(\Delta^{\mathrm{sel}}, v_{\bm{\eta}})
\qquad
\text{under the pre-selection Gaussian distribution.}
\]

Since
\[
\left\{
\mathcal{E}_{\bm{Y}}=\mathcal{E}_{\bm{y}},
\ \mathcal{Q}_{\bm{Y}}=\mathcal{Q}_{\bm{y}}
\right\}
=
\left\{
T_{\bm{\eta}}(\bm{Y})\in\mathcal{Z},
\ \mathcal{Q}_{\bm{Y}}=\mathcal{Q}_{\bm{y}}
\right\},
\]
the actual conditional distribution is obtained by restricting the pre-selection Gaussian density to the set \(T_{\bm{\eta}}(\bm{Y})\in\mathcal{Z}\).
Therefore, for \(t\in\mathbb{R}\),
\begin{align*}
&f_{\,T_{\bm{\eta}}(\bm{Y}) \mid \mathcal{E}_{\bm{Y}}=\mathcal{E}_{\bm{y}},\,\mathcal{Q}_{\bm{Y}}=\mathcal{Q}_{\bm{y}}}(t) \\
&\qquad =
\frac{
f^{\mathrm{pre}}_{\,T_{\bm{\eta}}(\bm{Y}) \mid \mathcal{Q}_{\bm{Y}}=\mathcal{Q}_{\bm{y}}}(t)\,
\mathbf{1}\{t\in\mathcal{Z}\}
}{
\displaystyle
\int_{\mathcal{Z}}
f^{\mathrm{pre}}_{\,T_{\bm{\eta}}(\bm{Y}) \mid \mathcal{Q}_{\bm{Y}}=\mathcal{Q}_{\bm{y}}}(u)\,du
}\\
&\qquad =
\frac{
f^{\mathrm{pre}}_{\,T_{\bm{\eta}}(\bm{Y})}(t)\mathbf{1}\{t\in\mathcal{Z}\}
}{
\displaystyle
\int_{\mathcal{Z}} f^{\mathrm{pre}}_{\,T_{\bm{\eta}}(\bm{Y})}(u)\,du
}\\
&\qquad =f_{\,\mathrm{TN}(\Delta^{\mathrm{sel}}, v_{\bm{\eta}}, \mathcal{Z})}(t).
\end{align*}
Finally, on the event \(\mathcal{E}_{\bm{Y}}=\mathcal{E}_{\bm{y}}\), we have \(\bm{\eta}_{\bm{Y}}=\bm{\eta}_{\bm{y}}=\bm{\eta}\), and hence
\[
T(\bm{Y})=T_{\bm{\eta}}(\bm{Y}).
\]
Therefore,
\[
T(\bm{Y}) \mid
\left\{
\mathcal{E}_{\bm{Y}}=\mathcal{E}_{\bm{y}},
\mathcal{Q}_{\bm{Y}}=\mathcal{Q}_{\bm{y}}
\right\}
\sim
\mathrm{TN}(\Delta^{\mathrm{sel}}, v_{\bm{\eta}}, \mathcal{Z}).
\]

This proves the claim.
\end{proof}

\subsection{Proof of Theorem~\ref{thm:validity}}
\label{app:proof_validity}

\begin{proof}
Fix the observed selection event and the sufficient statistic for the observable nuisance parameter.
By Lemma~\ref{lem:line_reduction}, the conditional distribution of \(T(\bm{Y})\) is a truncated normal distribution with mean \(\Delta^{\mathrm{sel}}\), variance \(v_{\bm{\eta}}\), and truncation set \(\mathcal{Z}\).
Under the null hypothesis \eqref{eq:post_bo_random_null}, we have \(\Delta^{\mathrm{sel}}\le 0\).
Let \(f_{\Delta,\mathcal{Z}}\) denote the corresponding density.

For \(\Delta_1>\Delta_0\) and \(t\in\mathcal{Z}\),
\begin{equation}
  \frac{f_{\Delta_1,\mathcal{Z}}(t)}{f_{\Delta_0,\mathcal{Z}}(t)}
  =
  C(\Delta_0,\Delta_1,\mathcal{Z})
  \exp\!\left(\frac{(\Delta_1-\Delta_0)t}{v_{\bm{\eta}}}\right),
\end{equation}
where \(C(\Delta_0,\Delta_1,\mathcal{Z})\) does not depend on \(t\).
Hence this family has a monotone likelihood ratio in the statistic \(t\), and is therefore stochastically increasing in \(\Delta\).
In particular, if
\(
\bar F_{\Delta,\mathcal{Z}}(t)=\Prob_\Delta\bigl(T\ge t\mid T\in\mathcal{Z}\bigr)
\)
denotes the truncated upper-tail probability, then
\begin{equation}
  \bar F_{\Delta,\mathcal{Z}}(t)
  \le
  \bar F_{0,\mathcal{Z}}(t)
  \qquad
  \text{for all } \Delta\le 0 \text{ and } t\in\mathcal{Z}.
  \label{eq:least_favorable_tail}
\end{equation}

The selective \(p\)-value in \eqref{eq:selective_pvalue} is exactly
\(
p_{\mathrm{sel}}(\bm{y})=\bar F_{0,\mathcal{Z}}(T(\bm{y}))
\),
that is, the upper-tail probability under the least-favorable null value \(\Delta^{\mathrm{sel}}=0\).

Now fix \(\alpha\in(0,1)\), and let \(c_\alpha\) satisfy
\begin{equation}
  \bar F_{0,\mathcal{Z}}(c_\alpha)=\alpha.
\end{equation}
Under the simple null hypothesis \(\Delta^{\mathrm{sel}}=0\), the conditional distribution of \(T(\bm{Y})\) is the truncated normal distribution indexed by \(\Delta=0\), so the probability-integral transform gives exact uniformity:
\begin{align}
  \Prob_{\Delta^{\mathrm{sel}}=0}
  \bigl(
    p_{\mathrm{sel}}(\bm{Y})\le \alpha
    \,\big|\,
    \mathcal{E}_{\bm{Y}}=\mathcal{E}_{\bm{y}},
    \mathcal{Q}_{\bm{Y}}=\mathcal{Q}_{\bm{y}}
  \bigr)
  &=
  \Prob_{0}\bigl(T(\bm{Y})\ge c_\alpha \mid T(\bm{Y})\in\mathcal{Z}\bigr)
  \notag\\
  &=
  \bar F_{0,\mathcal{Z}}(c_\alpha)
  =
  \alpha,
\end{align}
which proves the exact-uniformity claims in \eqref{eq:appendix_one_sided_conditional_exact} and \eqref{eq:appendix_one_sided_marginal_exact}.

For the composite null \(\Delta^{\mathrm{sel}}\le 0\), conditionally on the observed selection event and the sufficient statistic for the observable nuisance parameter,
\begin{align}
  \Prob_{\mathrm{H}_0}
  \bigl(
    p_{\mathrm{sel}}(\bm{Y})\le \alpha
    \,\big|\,
    \mathcal{E}_{\bm{Y}}=\mathcal{E}_{\bm{y}},
    \mathcal{Q}_{\bm{Y}}=\mathcal{Q}_{\bm{y}}
  \bigr)
  &=
  \Prob_\Delta\bigl(T(\bm{Y})\ge c_\alpha \mid T(\bm{Y})\in\mathcal{Z}\bigr)
  \notag\\
  &=
  \bar F_{\Delta,\mathcal{Z}}(c_\alpha)
  \notag\\
  &\le
  \bar F_{0,\mathcal{Z}}(c_\alpha)
  =
  \alpha,
\end{align}
where the inequality follows from \eqref{eq:least_favorable_tail}.
Taking expectations over the sufficient statistic for the observable nuisance parameter and the selection event preserves the inequality and yields
\begin{equation}
  \Prob_{\mathrm{H}_0}\bigl(p_{\mathrm{sel}}(\bm{Y})\le \alpha\bigr)\le \alpha.
\end{equation}
This proves \eqref{eq:validity_theorem}.
\end{proof}

\subsection[Proofs for confidence intervals]{Proofs for confidence intervals}
\label{app:proof_ci_validity}

\begin{proof}
Fix the observed values of \((\mathcal{E}_{\bm{Y}},\mathcal{Q}_{\bm{Y}})\).
By Lemma~\ref{lem:line_reduction}, the conditional distribution of \(T(\bm{Y})\) is a truncated normal family with mean parameter \(\Delta\), variance \(v_{\bm{\eta}}\), and truncation set \(\mathcal{Z}\).
As in the proof of Theorem~\ref{thm:validity}, this family has a monotone likelihood ratio in the statistic \(T\), hence is stochastically increasing in \(\Delta\).
Therefore, for fixed \(T(\bm{y})\), the conditional CDF \(G_\Delta^{\mathrm{sel}}(T(\bm{y}))\) is continuous and nonincreasing in \(\Delta\).
Since the conditional distribution shifts to the right as \(\Delta\) increases, we also have
\[
  G_\Delta^{\mathrm{sel}}(T(\bm{y})) \to 1
  \quad \text{as } \Delta\to -\infty,
  \qquad
  G_\Delta^{\mathrm{sel}}(T(\bm{y})) \to 0
  \quad \text{as } \Delta\to \infty.
\]
Hence the acceptance set in \eqref{eq:appendix_one_sided_ci} is a lower interval, and the acceptance set in \eqref{eq:selective_ci_delta} is an interval.

For the two-sided interval, the inversion identity in \eqref{eq:appendix_two_sided_ci_inversion} is immediate from the scalar equivalence
\[
  2\min\{u,1-u\}\ge \alpha
  \iff
  \frac{\alpha}{2}\le u\le 1-\frac{\alpha}{2}.
\]

Now evaluate that interval at the true effect \(\Delta=\Delta_{\bm{Y}}\).
Because the truncated normal distribution is continuous, the probability integral transform yields
\begin{equation}
  G_{\Delta_{\bm{Y}}}^{\mathrm{sel}}\!\left(T(\bm{Y})\right)
  \sim
  \mathrm{Unif}(0,1)
  \qquad
  \text{conditionally on } (\mathcal{E}_{\bm{Y}},\mathcal{Q}_{\bm{Y}}).
  \label{eq:appendix_uniform_transform}
\end{equation}
Hence
\begin{align}
  \Prob\!\left(
    \Delta_{\bm{Y}}\in C^{\Delta,+}_{1-\alpha}(\bm{Y})
    \,\middle|\,
    \mathcal{E}_{\bm{Y}},\mathcal{Q}_{\bm{Y}}
  \right)
  &=
  \Prob\!\left(
    G_{\Delta_{\bm{Y}}}^{\mathrm{sel}}\!\left(T(\bm{Y})\right)
    \le
    1-\alpha
    \,\middle|\,
    \mathcal{E}_{\bm{Y}},\mathcal{Q}_{\bm{Y}}
  \right)
  \notag\\
  &=
  1-\alpha,
\end{align}
which proves \eqref{eq:appendix_one_sided_ci_conditional}.
Likewise,
\begin{align}
  \Prob\!\left(
    \Delta_{\bm{Y}}\in C^{\Delta}_{1-\alpha}(\bm{Y})
    \,\middle|\,
    \mathcal{E}_{\bm{Y}},\mathcal{Q}_{\bm{Y}}
  \right)
  &=
  \Prob\!\left(
    \frac{\alpha}{2}
    \le
    G_{\Delta_{\bm{Y}}}^{\mathrm{sel}}\!\left(T(\bm{Y})\right)
    \le
    1-\frac{\alpha}{2}
    \,\middle|\,
    \mathcal{E}_{\bm{Y}},\mathcal{Q}_{\bm{Y}}
  \right)
  \notag\\
  &=
  1-\alpha,
\end{align}
which proves \eqref{eq:appendix_two_sided_conditional_delta}.

Taking expectations of \eqref{eq:appendix_one_sided_ci_conditional} and \eqref{eq:appendix_two_sided_conditional_delta} over \((\mathcal{E}_{\bm{Y}},\mathcal{Q}_{\bm{Y}})\) yields \eqref{eq:appendix_one_sided_ci_marginal} and \eqref{eq:appendix_two_sided_marginal}.
Under the point null \(\Delta^{\mathrm{sel}}=0\), \eqref{eq:appendix_uniform_transform} becomes
\(
G_0^{\mathrm{sel}}(T(\bm{Y}))\sim \mathrm{Unif}(0,1)
\)
conditionally on \((\mathcal{E}_{\bm{Y}},\mathcal{Q}_{\bm{Y}})\).
Therefore,
\begin{align}
  \Prob_{\Delta^{\mathrm{sel}}=0}\!\left(
    p_{\mathrm{sel}}^{\mathrm{ts}}(\bm{Y})\le \alpha
    \,\middle|\,
    \mathcal{E}_{\bm{Y}},\mathcal{Q}_{\bm{Y}}
  \right)
  &=
  \Prob_{\Delta^{\mathrm{sel}}=0}\!\left(
    G_0^{\mathrm{sel}}(T(\bm{Y}))\le \frac{\alpha}{2}
    \,\middle|\,
    \mathcal{E}_{\bm{Y}},\mathcal{Q}_{\bm{Y}}
  \right)
  \notag\\
  &\quad+
  \Prob_{\Delta^{\mathrm{sel}}=0}\!\left(
    G_0^{\mathrm{sel}}(T(\bm{Y}))\ge 1-\frac{\alpha}{2}
    \,\middle|\,
    \mathcal{E}_{\bm{Y}},\mathcal{Q}_{\bm{Y}}
  \right)
  \notag\\
  &=
  \alpha,
\end{align}
which proves \eqref{eq:appendix_two_sided_conditional_exact}.
Taking expectations yields \eqref{eq:appendix_two_sided_marginal_exact}.
The statement in Theorem~\ref{thm:ci_validity} is exactly the marginal coverage property in \eqref{eq:appendix_two_sided_marginal}.
\end{proof}

\newpage
\section{Computation of the Interval Set \(\mathcal{Z}\)}
\label{app:computation}

\subsection{A Toy Example}
\label{subsec:toy_example}

To give intuition for why the conditional distribution in our setting becomes truncated normal, we start with the smallest example that already mirrors the setting summarized in \S~\ref{sec:problem_setup} and \S~\ref{sec:target_of_inference}.
Fix \(\mathcal X=\{\bm{x}^{(1)},\bm{x}^{(2)},\bm{x}^{(3)}\}\) and a collection budget \(N=2\).
The first query is always \(\bm{x}_1=\bm{x}^{(1)}\), and for \(\mathcal D_1=\{(\bm x_1,Y_1)\}\) define
\[
  \mathcal A^{\rm ADC}(\mathcal D_1)
  =
  \begin{cases}
    \bm{x}^{(2)}, & Y_1<0,\\
    \bm{x}^{(3)}, & Y_1\ge 0.
  \end{cases}
\]
Equivalently, the queried trajectory is \((\bm{x}^{(1)},\bm{x}^{(2)})\) when \(Y_1<0\) and \((\bm{x}^{(1)},\bm{x}^{(3)})\) when \(Y_1\ge 0\).
This is the simplest example of a piecewise-constant ADC map in \S~\ref{sec:problem_setup}: for the fixed first query, the update is constant on the two halfspaces \(\{Y_1<0\}\) and \(\{Y_1\ge 0\}\).
As in \S~\ref{sec:selective_pvalues}, once the realized trajectory is fixed, the realized observed response vector is Gaussian,
\(
\bm{Y}=(Y_1,Y_2)^\top\sim N(\bm{\mu},I_2)
\),
where \(Y_2\) denotes the second response on the realized branch.

Next define the target-construction map \(G:\R^2\to\R^2\) by selecting the larger observed response:
\[
  \bm{\eta}_{\bm{Y}}
  =
  G(\bm{Y})
  =
  \begin{cases}
    (0,1)^\top, & Y_2\ge Y_1,\\
    (1,0)^\top, & Y_1>Y_2.
  \end{cases}
\]
Then the target and test statistic from \S~\ref{sec:target_of_inference} are
\(
\Delta_{\bm{Y}}=\bm{\eta}_{\bm{Y}}^\top\bm{\mu}
\)
and
\(
T(\bm{Y})=\bm{\eta}_{\bm{Y}}^\top\bm{Y}
\).
This is the simplest example of a piecewise-constant target-construction map in \S~\ref{sec:target_of_inference}, because \(G\) is constant on the two halfspaces \(\{Y_2\ge Y_1\}\) and \(\{Y_1>Y_2\}\).

Now we consider the conditional distribution
\(
  T_{\bm{\eta}}(\bm{Y})
  \ \Big|\
  \bigl\{
    \mathcal E_{\bm Y}=\mathcal E_{\bm y},
    \ \mathcal Q_{\bm Y}=\mathcal Q_{\bm y}
  \bigr\}
\)
in \eqref{eq:conditional_test_statistic_2}.

Because the pair \((\mathcal T_{\bm y},\bm{\eta}_{\bm y})\) has four possible realizations, the selection event has four branches, each of which corresponds to a different constraint on \(\bm{Y}\).
Figure~\ref{fig:toy_truncation_patterns} summarizes these four cases.
If \(\bm{\eta}=(0,1)^\top\), then \(T_{\bm{\eta}}(\bm{Y})=Y_2\) and \(\mathcal Q_{\bm Y}=(Y_1,0)^\top\), so conditioning on \(\mathcal Q_{\bm Y}=\mathcal Q_{\bm y}\) fixes \(Y_1=y_1\).
The trajectory condition then contributes only the already-fixed sign of \(y_1\), while the remaining event \(\bm{\eta}_{\bm Y}=\bm{\eta}\) is simply \(Y_2\ge y_1\).
Hence the admissible set is \(\mathcal Z=[y_1,\infty)\).
If \(\bm{\eta}=(1,0)^\top\), then \(T_{\bm{\eta}}(\bm{Y})=Y_1\) and \(\mathcal Q_{\bm Y}=(0,Y_2)^\top\), so conditioning fixes \(Y_2=y_2\).
The event \(\bm{\eta}_{\bm Y}=\bm{\eta}\) becomes \(Y_1>y_2\), and the trajectory condition contributes \(Y_1<0\) when the observed trajectory is \((\bm{x}^{(1)},\bm{x}^{(2)})\) and \(Y_1\ge 0\) when it is \((\bm{x}^{(1)},\bm{x}^{(3)})\).
Therefore \(\mathcal Z=(y_2,0)\) in the former case and \(\mathcal Z=[\max\{0,y_2\},\infty)\) in the latter.

\begin{figure}[t]
  \centering
  \begin{minipage}[t]{0.53\linewidth}
    \vspace{0pt}
    \centering
    \scriptsize
    \setlength{\tabcolsep}{2.5pt}
    \renewcommand{\arraystretch}{1.08}
    \resizebox{\linewidth}{!}{
    \begin{tabular}{@{}ccc@{}}
      \toprule
      & \(\bm{\eta}_{\bm y}=(0,1)^\top\) & \(\bm{\eta}_{\bm y}=(1,0)^\top\) \\
      \midrule
      \(\mathcal T_{\bm y}=(1,2)\) &
      \makecell[l]{\(Y_2\ge y_1\)\\ \(\Longleftrightarrow\)\\ \(\mathcal Z=[y_1,\infty)\)} &
      \makecell[l]{\(y_2<Y_1<0\)\\ \(\Longleftrightarrow\)\\ \(\mathcal Z=(y_2,0)\)} \\
      \(\mathcal T_{\bm y}=(1,3)\) &
      \makecell[l]{\(Y_2\ge y_1\)\\ \(\Longleftrightarrow\)\\ \(\mathcal Z=[y_1,\infty)\)} &
      \makecell[l]{\(Y_1\ge 0,\ Y_1>y_2\)\\ \(\Longleftrightarrow\)\\ \(\mathcal Z=[\max\{0,y_2\},\infty)\)} \\
      \bottomrule
    \end{tabular}
    }
  \end{minipage}
  \hfill
  \begin{minipage}[t]{0.43\linewidth}
    \vspace{0pt}
    \centering
    \includegraphics[width=1.0\textwidth,trim=0 0.2cm 0 0,clip]{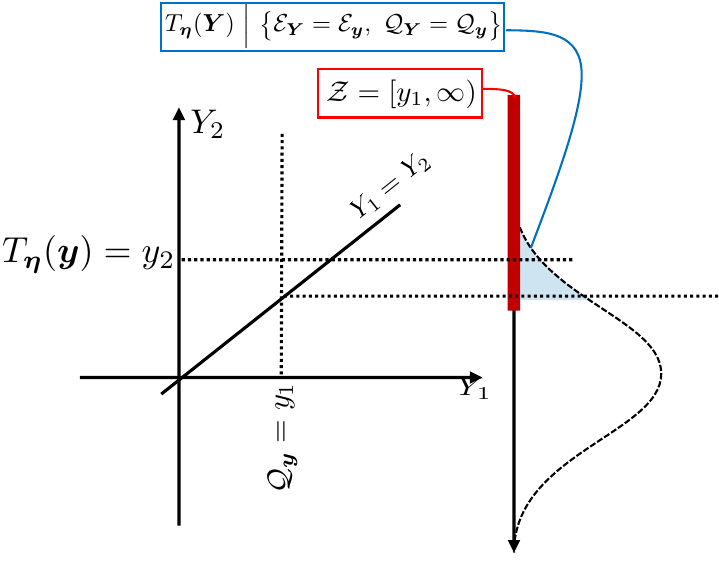}
  \end{minipage}
  \caption{Toy example summary. Left: The four possible realizations of the selection event by trajectory and selected weight vector, with each cell showing the corresponding constraint and the truncation set. Right: A schematic illustration for the branch \(\mathcal T_{\bm y}=(1,2)\) and \(\bm{\eta}_{\bm y}=(1,0)^\top\).}
  \label{fig:toy_truncation_patterns}
\end{figure}

Thus
\[
  T_{\bm{\eta}}(\bm{Y})
  \mid
  \bigl\{
    \mathcal E_{\bm Y}=\mathcal E_{\bm y},
    \ \mathcal Q_{\bm Y}=\mathcal Q_{\bm y}
  \bigr\}
  \sim
  \mathrm{TN}(\Delta^{\mathrm{sel}},1,\mathcal Z),
\]
which is a truncated normal distribution with mean \(\Delta^{\mathrm{sel}}=\bm{\eta}^\top\bm{\mu}\) and variance \(\bm{\eta}^\top I_2 \bm{\eta}=1\), truncated to the set \(\mathcal Z\) defined above.

This simple example gives intuition for the general case: after conditioning on the selection event and the sufficient statistic for the nuisance parameter, the conditional distribution of the test statistic becomes a one-dimensional Gaussian distribution truncated to a set \(\mathcal Z\) that can be represented by a finite collection of linear inequalities.

\subsection{A Sufficient Condition for Computing \(\mathcal{Z}\)}
\label{app:admissibility_framework}

Our framework enables the exact computation of \(\mathcal{Z}\) for a wide class of problems, including many widely used ADC algorithms and target-construction rules, provided that they satisfy the piecewise-constant properties in Definitions~\ref{def:ppc_adc} and~\ref{def:target_selector}.

\begin{definition}[Piecewise-constant ADC maps considered in this paper]
\label{def:ppc_adc}
We say that an ADC algorithm \(\mathcal A^{\rm ADC}\)
belongs to the class considered in this paper if, for every stage
\(n=1,\ldots,N-1\) and every fixed queried sequence
\(\bm x_{1:n}\in\mathcal X^n\), there exists a finite polyhedral partition
\(\{Q_{n,\ell}(\bm x_{1:n})\}_{\ell=1}^{L_n(\bm x_{1:n})}\)
of \(\mathbb R^n\) with associated points
\(\bm x_{n+1,\ell}\in\mathcal X\) such that, for each \(\ell\),
\[
\mathcal A^{\rm ADC}(\mathcal D_n)
=
\bm x_{n+1,\ell}
\qquad
\forall
\mathcal D_n\in(\mathcal X\times\mathbb R)^n
\quad
\mathrm{such\ that}
\quad
\bm y_{1:n}\in Q_{n,\ell}(\bm x_{1:n}).
\]
\end{definition}

\begin{definition}[Piecewise-constant target-construction maps considered in this paper]
\label{def:target_selector}
We say that a data-dependent target \(\Delta_{\bm Y}\) belongs to the target class considered in this paper if its weight vector is defined by a target-construction map
\[
  G: \R^N \to \R^N,
  \qquad
  \bm{Y} \mapsto \bm{\eta}_{\bm{Y}},
\]
that is polyhedral piecewise constant. That is, there exists a finite polyhedral partition \(\{P_k\}_{k=1}^K\) of \(\R^N\) such that, for each \(k\),
\[
  \bm{\eta}_{\bm{Y}} = \bm{c}_k
  \qquad
  \forall \bm{Y}\in P_k,
\]
for some vector \(\bm{c}_k\in\R^N\).
Data-independent maps are a special case of this definition.
\end{definition}

With these definitions in place, we introduce the trajectory map associated with an ADC algorithm,
\[
H_{\mathcal A}:\mathbb R^N \to \mathcal X^N,
\qquad
\bm Y \mapsto \mathcal T_{\bm Y},
\]
which returns the queried trajectory produced by the algorithm from the response vector.
When the ADC algorithm in \S~\ref{sec:problem_setup} is polyhedral piecewise constant
in the sense of Definition~\ref{def:ppc_adc},
the map \(H_{\mathcal A}\) is itself polyhedral piecewise constant.
This geometric property is the key to the proof below.

\begin{definition}[Problems considered in this paper]
\label{def:admissible_problem}
A post-ADC analysis problem belongs to the class considered in this paper if its target of inference belongs to the target class in Definition~\ref{def:target_selector}, its ADC algorithm satisfies Definition~\ref{def:ppc_adc}, and the weight vector \(\bm{\eta}_{\bm{Y}}\) is nonzero.
\end{definition}

\begin{proposition}[Exact computation of the interval set \(\mathcal{Z}\) for problems considered in this paper]
  \label{prop:admissibility_interval_set}
  Fix an observed response vector \(\bm{y}\) with \(\bm{\eta}_{\bm{y}}\neq \bm{0}\), and define the one-dimensional slice \(\bm{Y}(z)=\bm{a}+\bm{b}z\) as in \eqref{eq:one_dimensional_slice}.
  If the problem belongs to the class considered in this paper with target-construction map \(G\) and ADC algorithm \(\mathcal A\), then
  \[
    \mathcal{Z}_{G,\mathcal A}
    =
    \left\{
      z\in\R:
      G(\bm{Y}(z))=G(\bm{y}),
      \
      H_{\mathcal A}(\bm{Y}(z))=H_{\mathcal A}(\bm{y})
    \right\}
  \]
  is a finite union of intervals obtained from finitely many linear equalities and inequalities in \(z\).
  Consequently, \(\mathcal{Z}_{G,\mathcal A}\) can be computed exactly by solving finitely many linear constraints.
\end{proposition}

\begin{proof}
For each target-construction cell \(P_k\), the event
\[
  \{ \bm{Y}\in P_k,\ G(\bm{Y})=G(\bm{y}) \}
\]
is a polyhedron because \(P_k\) is polyhedral and \(G\) is constant on \(P_k\).
Taking the union over finitely many cells shows that the target-construction event
\[
  \{G(\bm{Y})=G(\bm{y})\}
\]
is a finite union of polyhedra.
By the same reasoning, the trajectory event
\[
  \{H_{\mathcal A}(\bm{Y})=H_{\mathcal A}(\bm{y})\}
\]
is a finite union of polyhedra.
Their intersection is therefore a finite union of polyhedra in \(\R^N\).
Intersecting these polyhedra with the line
\[
  \{\bm{a}+\bm{b}z:\ z\in\R\}
\]
produces a finite union of intervals on the \(z\)-axis, each described by linear equalities and inequalities in \(z\).
\end{proof}

\subsection{Algorithm-Specific Interval Computation Details}
\label{app:interval_computation}

This subsection provides more details on how to compute the interval set \(\mathcal{Z}\) for the GP-UCB and TPE algorithms.

To compute the interval \([L_z,U_z]\) around a seed point \(z\in\R\), it is enough to derive conditions under which applying the same ADC algorithm recursively to \((\mathcal{X},\bm{a}+\bm{b}r)\) yields the same observed selection event as applying it to \((\mathcal{X},\bm{a}+\bm{b}z)\).
In the main-text setting, the regions \((\mathcal{I}_{\bm{a}+\bm{b}r},\mathcal{J}_{\bm{a}+\bm{b}r})\) are determined by order relations among the observed responses, so once the queried trajectory is fixed they can be represented by linear inequalities in \(r\).

The remaining challenge is therefore to characterize when the queried trajectory and the acquisition-score comparisons from the original ADC run are unchanged.
Suppose that the first \(n\) ADC steps agree with that run.
Then \(\bm{X}_n\) is fixed, and the next queried point \(\bm{x}_{n+1}\) is reproduced whenever
\[
  \mathcal A^{\rm ADC}(\mathcal D_n(r))=\bm x_{n+1}.
\]
For acquisition-maximizing ADC algorithms, this is equivalent to
\begin{equation}
  \forall \bm{x}\in\mathcal{X},\
  a(\bm{x};\mathcal{D}_n(r))
  \le
  a(\bm{x}_{n+1};\mathcal{D}_n(r)),
  \label{eq:appendix_next_point_condition}
\end{equation}
where \(\mathcal{D}_n(r)=\{(\bm{x}_i,y_i(r))\}_{i=1}^n\) is the restricted dataset formed from the responses
\(
\bm{y}_{1:n}(r)=\bm{a}_n+\bm{b}_n r
\).

\paragraph{GP-UCB.}
For GP-UCB, the posterior mean has the form
\[
  \mu_n(\bm{x};\bm{y}_{1:n}(r))
  =
  \bm{k}_n(\bm{x})^\top (K_n+\sigma^2 I_n)^{-1} \bm{y}_{1:n}(r),
\]
which is affine in the restricted response vector \(\bm{y}_{1:n}(r)=\bm{a}_n+\bm{b}_n r\).
The posterior variance \(s_n^2(\bm{x})\) depends only on the queried inputs \(\bm{x}_{1:n}\) and is therefore fixed once the history is fixed.
Hence, for every candidate \(\bm{x}\in\mathcal X\), the GP-UCB comparison
\[
  \mu_n(\bm{x};\bm{y}_{1:n}(r)) + \kappa s_n(\bm{x})
  \le
  \mu_n(\bm{x}_{n+1};\bm{y}_{1:n}(r)) + \kappa s_n(\bm{x}_{n+1})
\]
is a linear inequality in \(r\).
Intersecting these inequalities over all \(\bm{x}\in\mathcal X\) and over stages \(n=1,\ldots,N-1\) yields a finite family of linear inequalities that characterizes the GP-UCB part of the selection-event-consistent set along the slice.

\paragraph{TPE.}
Fix \(m_n=\lceil \gamma n\rceil\), and let \(\mathcal{H}_n(\bm{y})\) and \(\mathcal{L}_n(\bm{y})\) denote the upper and lower TPE partitions induced by \(\bm{y}\).
To keep the same TPE acquisition comparison on the slice in \eqref{eq:one_dimensional_slice}, we condition on the quantile-partition event
\begin{equation}
  y_i(r)\ge y_j(r)
  \qquad
  \forall i\in\mathcal{H}_n(\bm{y}),
  \ \forall j\in\mathcal{L}_n(\bm{y}).
  \label{eq:appendix_tpe_partition}
\end{equation}
Because each \(y_i(r)\) has the form \(a_i+b_i r\), \eqref{eq:appendix_tpe_partition} is again a finite collection of linear inequalities.
Up to the probability-zero tie boundary, this conditioning is equivalent to requiring
\(
\mathcal{H}_n(\bm{y}(r))=\mathcal{H}_n(\bm{y})
\)
and
\(
\mathcal{L}_n(\bm{y}(r))=\mathcal{L}_n(\bm{y})
\).

Once the partition is fixed, the TPE score no longer depends on the response magnitudes.
Indeed, define
\begin{equation}
  \begin{aligned}
  \bar g_n(\bm{x})
  &:=
  \frac{1}{|\mathcal{H}_n(\bm{y})|}
  \sum_{i\in\mathcal{H}_n(\bm{y})}
  \kappa_{\mathrm{TPE}}(\bm{x},\bm{x}_i),\\
  \bar \ell_n(\bm{x})
  &:=
  \frac{1}{|\mathcal{L}_n(\bm{y})|}
  \sum_{i\in\mathcal{L}_n(\bm{y})}
  \kappa_{\mathrm{TPE}}(\bm{x},\bm{x}_i).
  \end{aligned}
  \label{eq:appendix_tpe_constants}
\end{equation}
Then for every \(r\) satisfying \eqref{eq:appendix_tpe_partition},
\begin{equation}
  a^{\mathrm{TPE}}(\bm{x};\mathcal D_n(r))
  =
  \frac{\bar g_n(\bm{x})}{\bar \ell_n(\bm{x})},
\end{equation}
which is constant in \(r\).
Since \(\kappa_{\mathrm{TPE}}>0\), all denominators are positive, and the next-point condition
\[
  a^{\mathrm{TPE}}(\bm{x};\mathcal D_n(r))
  \le
  a^{\mathrm{TPE}}(\bm{x}_{n+1};\mathcal D_n(r))
\]
is equivalent to the constant inequality
\begin{equation}
  \bar g_n(\bm{x})\bar \ell_n(\bm{x}_{n+1})
  \le
  \bar g_n(\bm{x}_{n+1})\bar \ell_n(\bm{x}).
\end{equation}
Thus the TPE contribution to the truncation set is determined by linear quantile-partition constraints.
It is combined with fixed score comparisons on the finite candidate set.

\newpage
\section{Randomized-Response Extension}
\label{app:randomized_extension}

This appendix presents a randomized-response variant of the selective distribution in our setting.
The randomization replaces the hard-truncation distribution on \(\mathcal{Z}\) by a randomized distribution on \(\widetilde{\mathcal{Z}}\), stabilizes the resulting conditional distribution, and avoids the interval-length instability studied in randomized selective inference \citep{tian2018selective, kivaranovic2024tight}.

\subsection{Randomized Selection and Conditional Distribution}

Throughout this subsection, assume that both the response covariance and the randomization covariance are isotropic:
\[
  \Sigma=\sigma^2 I_N,
  \qquad
  \Gamma=\tau^2 I_N.
\]
Let
\[
  \bm{\Omega}\sim\mathcal{N}(\bm{0},\Gamma)
\]
be independent of \(\bm{Y}\), where \(\tau^2\ge 0\) is a known randomization level, and define the randomized response vector
\[
  \widetilde{\bm{Y}}=\bm{Y}+\bm{\Omega}.
\]
The optimizer and target-construction map are now applied to \(\widetilde{\bm{Y}}\), whereas inference is still carried out for the original responses \(\bm{Y}\).
Let
\[
  \widetilde{\bm{y}}
  =
  \bm{y}+\bm{\omega}
\]
denote the randomized response vector corresponding to the realization \((\bm{y},\bm{\omega})\), where \(\bm{y}\) and \(\bm{\omega}\) are realizations of \(\bm{Y}\) and \(\bm{\Omega}\), respectively, and fix the resulting weight vector
\(
  \bm{\eta}=\bm{\eta}_{\widetilde{\bm{y}}}
\)
and corresponding effect
\(
  \Delta^{\mathrm{sel,rnd}}=\bm{\eta}^\top \bm{\mu}
\).
As in the unrandomized case, let
\[
  v_{\bm{\eta}}
  =
  \bm{\eta}^\top \Sigma \bm{\eta},
  \qquad
  \tau_{\bm{\eta}}^2
  =
  \bm{\eta}^\top \Gamma \bm{\eta}
  =
  \tau^2\|\bm{\eta}\|_2^2.
\]

\begin{definition}[Randomized selection event and the sufficient statistic for the nuisance parameter]
  \label{def:appendix_randomized_selection}
  Define
  \begin{equation}
    \widetilde{\mathcal{E}}_{\widetilde{\bm{Y}}}
    =
    \bigl(
      \mathcal{T}_{\widetilde{\bm{Y}}},
      \bm{\eta}_{\widetilde{\bm{Y}}}
    \bigr),
    \label{eq:appendix_randomized_selection_event}
  \end{equation}
  with observed value
  \(
    \widetilde{\mathcal{E}}_{\widetilde{\bm{y}}}
    =
    \bigl(
      \mathcal{T}_{\widetilde{\bm{y}}},
      \bm{\eta}_{\widetilde{\bm{y}}}
    \bigr)
  \),
  and
  \begin{equation}
    \widetilde{\mathcal{Q}}_{\widetilde{\bm{Y}}}
    =
    \left(
      I_N
      -
      \frac{(\Sigma+\Gamma)\bm{\eta}\bm{\eta}^\top}
           {\bm{\eta}^\top(\Sigma+\Gamma)\bm{\eta}}
    \right)\widetilde{\bm{Y}}
    =
    \left(
      I_N
      -
      \frac{\bm{\eta}\bm{\eta}^\top}
           {\bm{\eta}^\top \bm{\eta}}
    \right)\widetilde{\bm{Y}}.
    \label{eq:appendix_randomized_nuisance}
  \end{equation}
\end{definition}

The statistic used for inference remains the weight vector in the unrandomized responses,
\[
  T(\bm{Y})=\bm{\eta}^\top \bm{Y},
\]
with the weight vector fixed by the randomized selection.
Its observed value is \(T(\bm{y})=\bm{\eta}^\top \bm{y}\).

\begin{proposition}[Randomized conditional distribution and CDF]
  \label{prop:appendix_randomized_distribution}
  Suppose the selection map belongs to the class considered in Appendix~\ref{app:admissibility_framework} when viewed as a function of \(\widetilde{\bm{Y}}\).
Define
  \begin{equation}
    \widetilde{\bm{a}}
    =
    \widetilde{\mathcal{Q}}_{\widetilde{\bm{y}}},
    \qquad
    \widetilde{\bm{b}}
    =
    \frac{(\Sigma+\Gamma)\bm{\eta}}
         {\bm{\eta}^\top(\Sigma+\Gamma)\bm{\eta}}
    =
    \frac{\bm{\eta}}
         {\bm{\eta}^\top \bm{\eta}},
    \qquad
    \widetilde{\bm{Y}}(z)
    =
    \widetilde{\bm{a}}+\widetilde{\bm{b}}z,
    \label{eq:appendix_randomized_line}
  \end{equation}
  and let
  \begin{equation}
    \widetilde{\mathcal{Z}}
    =
    \left\{
      z\in\R:
      \widetilde{\mathcal{E}}_{\widetilde{\bm{Y}}(z)}
      =
      \widetilde{\mathcal{E}}_{\widetilde{\bm{y}}}
    \right\}.
    \label{eq:appendix_randomized_truncation_set}
  \end{equation}
  Then \(\widetilde{\mathcal{Z}}\) is a finite union of intervals.
Moreover, if
  \[
    Z\sim\mathcal{N}(\Delta^{\mathrm{sel,rnd}}, v_{\bm{\eta}}),
    \qquad
    R\sim\mathcal{N}(0,\tau_{\bm{\eta}}^2)
  \]
  are independent, then
  \begin{equation}
    \mathcal{L}\!\left(
      T(\bm{Y})
      \,\middle|\,
      \widetilde{\mathcal{E}}_{\widetilde{\bm{Y}}}=\widetilde{\mathcal{E}}_{\widetilde{\bm{y}}},
      \widetilde{\mathcal{Q}}_{\widetilde{\bm{Y}}}=\widetilde{\mathcal{Q}}_{\widetilde{\bm{y}}}
    \right)
    =
    \mathcal{L}\!\left(
      Z
      \,\middle|\,
      Z+R\in\widetilde{\mathcal{Z}}
    \right).
    \label{eq:appendix_randomized_distribution}
  \end{equation}
  Equivalently, writing
  \[
    \phi_{\mu,s^2}(x)
    :=
    \frac{1}{\sqrt{2\pi s^2}}
    \exp\!\left(
      -\frac{(x-\mu)^2}{2s^2}
    \right)
  \]
  for the \(\mathcal{N}(\mu,s^2)\) density and
  \[
    w_{\widetilde{\mathcal{Z}}}(t)
    :=
    \int_{\R}
    \phi_{0,\tau_{\bm{\eta}}^2}(r)\,
    \mathbf{1}\{t+r\in\widetilde{\mathcal{Z}}\}\,dr,
  \]
  then for every \(\Delta\in\R\), if
  \(
    Z_\Delta\sim\mathcal{N}(\Delta,v_{\bm{\eta}})
  \)
  is independent of \(R\), the corresponding randomized selective CDF is
  \begin{equation}
    G^{\mathrm{rnd}}_{\Delta,\widetilde{\mathcal{Z}}}(t)
    :=
    \Prob\!\left(
      Z_\Delta\le t
      \,\middle|\,
      Z_\Delta+R\in\widetilde{\mathcal{Z}}
    \right)
    =
    \frac{
      \displaystyle
      \int_{-\infty}^{t}
      \phi_{\Delta,v_{\bm{\eta}}}(u)\,
      w_{\widetilde{\mathcal{Z}}}(u)\,du
    }{
      \displaystyle
      \int_{\R}
      \phi_{\Delta,v_{\bm{\eta}}}(u)\,
      w_{\widetilde{\mathcal{Z}}}(u)\,du
    }.
    \label{eq:appendix_randomized_distribution_cdf}
  \end{equation}
\end{proposition}

\begin{proof}
Let
\(
\widetilde{\bm{Y}}=\bm{Y}+\bm{\Omega}
\)
and fix the resulting weight vector
\(
\bm{\eta}=\bm{\eta}_{\widetilde{\bm{y}}}
\).
\[
  U
  :=
  \bm{\eta}^\top \widetilde{\bm{Y}}
  =
  \bm{\eta}^\top \bm{Y}
  +
  \bm{\eta}^\top \bm{\Omega}
  =
  T(\bm{Y})+R,
\]
where
\(
R=\bm{\eta}^\top \bm{\Omega}\sim\mathcal{N}(0,\tau_{\bm{\eta}}^2)
\)
is independent of
\(
T(\bm{Y})=\bm{\eta}^\top \bm{Y}\sim\mathcal{N}(\Delta^{\mathrm{sel,rnd}},v_{\bm{\eta}})
\).
Because \(\Sigma=\sigma^2 I_N\) and \(\Gamma=\tau^2 I_N\), we have
\[
  \Sigma+\Gamma=(\sigma^2+\tau^2)I_N,
\]
so the projector in \eqref{eq:appendix_randomized_nuisance} reduces to
\[
  P_{\bm{\eta}}
  :=
  I_N
  -
  \frac{(\Sigma+\Gamma)\bm{\eta}\bm{\eta}^\top}
       {\bm{\eta}^\top(\Sigma+\Gamma)\bm{\eta}}
  =
  I_N
  -
  \frac{\bm{\eta}\bm{\eta}^\top}
       {\bm{\eta}^\top\bm{\eta}},
\]
and \(\widetilde{\mathcal{Q}}_{\widetilde{\bm{Y}}}=P_{\bm{\eta}}\widetilde{\bm{Y}}\).
The direction vector in \eqref{eq:appendix_randomized_line} likewise reduces to its Euclidean form.
Also,
\[
  \widetilde{\bm{Y}}-\widetilde{\mathcal{Q}}_{\widetilde{\bm{Y}}}
  =
  (I_N-P_{\bm{\eta}})\widetilde{\bm{Y}}
  =
  \widetilde{\bm{b}}\,U.
\]
Hence
\[
  \widetilde{\bm{Y}}
  =
  \widetilde{\mathcal{Q}}_{\widetilde{\bm{Y}}}
  +
  \widetilde{\bm{b}}\,U.
\]
Conditioning on
\(
\widetilde{\mathcal{Q}}_{\widetilde{\bm{Y}}}
=
\widetilde{\mathcal{Q}}_{\widetilde{\bm{y}}}
\)
therefore yields the one-dimensional representation
\[
  \widetilde{\bm{Y}}
  =
  \widetilde{\bm{a}}+\widetilde{\bm{b}}u,
  \qquad
  u=U,
\]
with \(\widetilde{\bm{a}}\) and \(\widetilde{\bm{b}}\) as in \eqref{eq:appendix_randomized_line}.
Since the selection map belongs to the class considered in Appendix~\ref{app:admissibility_framework} as a function of \(\widetilde{\bm{Y}}\), intersecting its polyhedral cells with this line shows that \(\widetilde{\mathcal{Z}}\) in \eqref{eq:appendix_randomized_truncation_set} is a finite union of intervals.
Therefore,
\[
  \widetilde{\mathcal{E}}_{\widetilde{\bm{Y}}}
  =
  \widetilde{\mathcal{E}}_{\widetilde{\bm{y}}}
  \iff
  U\in\widetilde{\mathcal{Z}}
  \iff
  T(\bm{Y})+R\in\widetilde{\mathcal{Z}},
\]
under the condition
\(
\widetilde{\mathcal{Q}}_{\widetilde{\bm{Y}}}
=
\widetilde{\mathcal{Q}}_{\widetilde{\bm{y}}}
\).

\begin{align*}
  \operatorname{Cov}\!\left(
    T(\bm{Y}),
    \widetilde{\mathcal{Q}}_{\widetilde{\bm{Y}}}
  \right)
  &=
  \operatorname{Cov}\!\left(
    \bm{\eta}^\top\bm{Y},
    P_{\bm{\eta}}(\bm{Y}+\bm{\Omega})
  \right) \\
  &=
  \bm{\eta}^\top \Sigma P_{\bm{\eta}}^\top \\
  &=
  \bm{0}^\top,
  \\
  \operatorname{Cov}\!\left(
    R,
    \widetilde{\mathcal{Q}}_{\widetilde{\bm{Y}}}
  \right)
  &=
  \operatorname{Cov}\!\left(
    \bm{\eta}^\top\bm{\Omega},
    P_{\bm{\eta}}(\bm{Y}+\bm{\Omega})
  \right) \\
  &=
  \bm{\eta}^\top \Gamma P_{\bm{\eta}}^\top \\
  &=
  \bm{0}^\top.
\end{align*}
Thus the jointly Gaussian pair \(\bigl(T(\bm{Y}),R\bigr)\) is independent of \(\widetilde{\mathcal{Q}}_{\widetilde{\bm{Y}}}\). Hence
\begin{align}
  &\mathcal{L}\!\left(
    T(\bm{Y})
    \,\middle|\,
    \widetilde{\mathcal{E}}_{\widetilde{\bm{Y}}}
    =
    \widetilde{\mathcal{E}}_{\widetilde{\bm{y}}},
    \widetilde{\mathcal{Q}}_{\widetilde{\bm{Y}}}
    =
    \widetilde{\mathcal{Q}}_{\widetilde{\bm{y}}}
  \right)
  \notag\\
  &\qquad=
  \mathcal{L}\!\left(
    T(\bm{Y})
    \,\middle|\,
    T(\bm{Y})+R\in\widetilde{\mathcal{Z}},
    \widetilde{\mathcal{Q}}_{\widetilde{\bm{Y}}}
    =
    \widetilde{\mathcal{Q}}_{\widetilde{\bm{y}}}
  \right)
  \notag\\
  &\qquad=
  \mathcal{L}\!\left(
    T(\bm{Y})
    \,\middle|\,
    T(\bm{Y})+R\in\widetilde{\mathcal{Z}}
  \right).
\end{align}
Replacing \(T(\bm{Y})\) by an equal-distribution copy \(Z\sim\mathcal{N}(\Delta^{\mathrm{sel,rnd}},v_{\bm{\eta}})\), independent of \(R\), gives \eqref{eq:appendix_randomized_distribution}.
At this stage the conditioning on \(\widetilde{\mathcal{Q}}_{\widetilde{\bm{Y}}}\) has already been removed, so no independence between \(Z\) and \(\widetilde{\mathcal{Q}}_{\widetilde{\bm{Y}}}\) is needed; \(Z\) is only a copy of \(T(\bm{Y})\) used to represent the resulting one-dimensional distribution.
Now fix \(\Delta\in\R\), let \(Z_\Delta\sim\mathcal{N}(\Delta,v_{\bm{\eta}})\) be independent of \(R\), and let \(t\in\R\). By Bayes' rule,
\begin{align*}
  G^{\mathrm{rnd}}_{\Delta,\widetilde{\mathcal{Z}}}(t)
  &=
  \Prob\!\left(
    Z_\Delta\le t
    \,\middle|\,
    Z_\Delta+R\in\widetilde{\mathcal{Z}}
  \right) \\
  &=
  \frac{
    \Prob\!\left(
      Z_\Delta\le t,\,
      Z_\Delta+R\in\widetilde{\mathcal{Z}}
    \right)
  }{
    \Prob\!\left(
      Z_\Delta+R\in\widetilde{\mathcal{Z}}
    \right)
  }\\
  &=
  \frac{
    \displaystyle
    \int_{-\infty}^{t}
    \phi_{\Delta,v_{\bm{\eta}}}(u)\,
    w_{\widetilde{\mathcal{Z}}}(u)\,du
  }{
    \displaystyle
    \int_{\R}
    \phi_{\Delta,v_{\bm{\eta}}}(u)\,
    w_{\widetilde{\mathcal{Z}}}(u)\,du
  }.
\end{align*}
This proves \eqref{eq:appendix_randomized_distribution_cdf}. Differentiating in \(t\) yields the corresponding density
\[
  f^{\mathrm{rnd}}_{\Delta,\widetilde{\mathcal{Z}}}(t)
  =
  \frac{
    \phi_{\Delta,v_{\bm{\eta}}}(t)\,
    w_{\widetilde{\mathcal{Z}}}(t)
  }{
    \displaystyle
    \int_{\R}
    \phi_{\Delta,v_{\bm{\eta}}}(u)\,
    w_{\widetilde{\mathcal{Z}}}(u)\,du
  }.
\]
\end{proof}

\subsection{Randomized Inferential Outputs}

For the observed selection event and the sufficient statistic for the nuisance parameter associated with \(\widetilde{\bm{Y}}\), write
\begin{equation}
  G_{\Delta}^{\mathrm{sel,rnd}}(t)
  =
  \Prob_{\Delta^{\mathrm{sel,rnd}}=\Delta}\!\left(
    T(\bm{Y})\le t
    \,\middle|\,
    \widetilde{\mathcal{E}}_{\widetilde{\bm{Y}}}
    =
    \widetilde{\mathcal{E}}_{\widetilde{\bm{y}}},
    \
    \widetilde{\mathcal{Q}}_{\widetilde{\bm{Y}}}
    =
    \widetilde{\mathcal{Q}}_{\widetilde{\bm{y}}}
  \right),
  \label{eq:appendix_randomized_cdf}
\end{equation}
By Proposition~\ref{prop:appendix_randomized_distribution}, this is exactly the CDF in \eqref{eq:appendix_randomized_distribution_cdf} for the observed selection-consistent set \(\widetilde{\mathcal{Z}}\).

\begin{definition}[Randomized selective \(p\)-value and confidence interval]
  \label{def:appendix_randomized_outputs}
  The randomized one-sided selective \(p\)-value is
  \begin{equation}
    p_{\mathrm{sel}}^{\mathrm{rnd}}(\bm{y},\bm{\omega})
    =
    \Prob_{\Delta^{\mathrm{sel,rnd}}=0}\!\left(
      T(\bm{Y})\ge T(\bm{y})
      \,\middle|\,
      \widetilde{\mathcal{E}}_{\widetilde{\bm{Y}}}=\widetilde{\mathcal{E}}_{\widetilde{\bm{y}}},
      \
      \widetilde{\mathcal{Q}}_{\widetilde{\bm{Y}}}
      =
      \widetilde{\mathcal{Q}}_{\widetilde{\bm{y}}}
    \right),
    \label{eq:appendix_randomized_pvalue}
  \end{equation}
  and the corresponding equal-tailed randomized selective confidence interval is
  \begin{equation}
    C_{1-\alpha}^{\Delta,\mathrm{rnd}}
    (\bm{y},\bm{\omega})
    =
    \left\{
      \Delta\in\R:
      \frac{\alpha}{2}
      \le
      G_{\Delta}^{\mathrm{sel,rnd}}\!\left(T(\bm{y})\right)
      \le
      1-\frac{\alpha}{2}
    \right\}.
    \label{eq:appendix_randomized_ci}
  \end{equation}
  When \(\tau^2=0\), these expressions reduce to \eqref{eq:selective_pvalue} and \eqref{eq:selective_ci_delta}.
\end{definition}

\begin{proposition}[Validity of the randomized selective procedures]
  \label{prop:appendix_randomized_validity}
  Under the assumptions of Proposition~\ref{prop:appendix_randomized_distribution}, the induced randomized selective \(p\)-value is exactly uniform under the simple null hypothesis and valid for the composite null.
In particular, for every \(\alpha\in(0,1)\),
  \begin{equation}
    \Prob_{\Delta^{\mathrm{sel,rnd}}=0}\!\left(
      p_{\mathrm{sel}}^{\mathrm{rnd}}(\bm{Y},\bm{\Omega})\le\alpha
      \,\middle|\,
      \widetilde{\mathcal{E}}_{\widetilde{\bm{Y}}},
      \widetilde{\mathcal{Q}}_{\widetilde{\bm{Y}}}
    \right)
    =
    \alpha
    \qquad \text{a.s.}
    \label{eq:appendix_randomized_pvalue_exact}
  \end{equation}
  and therefore
  \begin{equation}
    \Prob_{\Delta^{\mathrm{sel,rnd}}=0}\!\left(
      p_{\mathrm{sel}}^{\mathrm{rnd}}(\bm{Y},\bm{\Omega})\le\alpha
    \right)
    =
    \alpha.
    \label{eq:appendix_randomized_pvalue_marginal_exact}
  \end{equation}
  Consequently,
  \begin{equation}
    \Prob_{\mathrm{H}_0}\!\left(
      p_{\mathrm{sel}}^{\mathrm{rnd}}(\bm{Y},\bm{\Omega})\le\alpha
      \,\middle|\,
      \widetilde{\mathcal{E}}_{\widetilde{\bm{Y}}},
      \widetilde{\mathcal{Q}}_{\widetilde{\bm{Y}}}
    \right)
    \le
    \alpha
    \qquad \text{a.s.}
    \label{eq:appendix_randomized_pvalue_validity}
  \end{equation}
  and therefore
  \begin{equation}
    \Prob_{\mathrm{H}_0}\!\left(
      p_{\mathrm{sel}}^{\mathrm{rnd}}(\bm{Y},\bm{\Omega})\le\alpha
    \right)
    \le
    \alpha.
    \label{eq:appendix_randomized_pvalue_marginal}
  \end{equation}
  Writing
  \[
    \Delta_{\bm{Y},\bm{\Omega}}^{\mathrm{rnd}}
    =
    \bm{\eta}_{\bm{Y}+\bm{\Omega}}^\top \bm{\mu}
  \]
  for the effect induced by the randomized responses, we also have
  \begin{equation}
    \Prob\!\left(
      \Delta_{\bm{Y},\bm{\Omega}}^{\mathrm{rnd}}
      \in
      C_{1-\alpha}^{\Delta,\mathrm{rnd}}(\bm{Y},\bm{\Omega})
      \,\middle|\,
      \widetilde{\mathcal{E}}_{\widetilde{\bm{Y}}},
      \widetilde{\mathcal{Q}}_{\widetilde{\bm{Y}}}
    \right)
    =
    1-\alpha
    \qquad \text{a.s.}
    \label{eq:appendix_randomized_ci_validity}
  \end{equation}
  and hence
  \begin{equation}
    \Prob\!\left(
      \Delta_{\bm{Y},\bm{\Omega}}^{\mathrm{rnd}}
      \in
      C_{1-\alpha}^{\Delta,\mathrm{rnd}}(\bm{Y},\bm{\Omega})
    \right)
    =
    1-\alpha.
    \label{eq:appendix_randomized_ci_marginal}
  \end{equation}
\end{proposition}

\begin{proof}
For the validity claim, fix the observed selection event and the sufficient statistic for the nuisance parameter associated with \(\widetilde{\bm{Y}}\).
By \eqref{eq:appendix_randomized_distribution}, the conditional density of \(T(\bm{Y})\) under parameter value \(\Delta\) is proportional to
\begin{equation}
  \exp\!\left(
    -\frac{(t-\Delta)^2}{2v_{\bm{\eta}}}
  \right)
  \Prob\!\left(
    R\in\widetilde{\mathcal{Z}}-t
  \right),
  \qquad
  \widetilde{\mathcal{Z}}-t
  :=
  \left\{
    r\in\R:\ t+r\in\widetilde{\mathcal{Z}}
  \right\}.
  \label{eq:appendix_randomized_density}
\end{equation}
The factor
\(
\Prob(R\in\widetilde{\mathcal{Z}}-t)
\)
does not depend on \(\Delta\).
Hence, for \(\Delta_1>\Delta_0\),
\begin{equation}
  \frac{
    f^{\mathrm{rnd}}_{\Delta_1,\widetilde{\mathcal{Z}}}(t)
  }{
    f^{\mathrm{rnd}}_{\Delta_0,\widetilde{\mathcal{Z}}}(t)
  }
  =
  C(\Delta_0,\Delta_1,\widetilde{\mathcal{Z}})
  \exp\!\left(
    \frac{(\Delta_1-\Delta_0)t}{v_{\bm{\eta}}}
  \right),
  \label{eq:appendix_randomized_mlr}
\end{equation}
where \(C(\Delta_0,\Delta_1,\widetilde{\mathcal{Z}})\) is constant in \(t\).
Thus the randomized family has a monotone likelihood ratio in \(T\), so it is stochastically increasing in \(\Delta\).
The same least-favorable-null-value argument used in the proof of Theorem~\ref{thm:validity} therefore gives \eqref{eq:appendix_randomized_pvalue_validity} and \eqref{eq:appendix_randomized_pvalue_marginal}.

For the confidence interval, the monotone likelihood ratio property implies that
\(
G_{\Delta}^{\mathrm{sel,rnd}}(T(\bm{y}))
\)
is monotone in \(\Delta\), so the acceptance set in \eqref{eq:appendix_randomized_ci} is an interval.
Evaluating that interval at the true randomized effect
\(
\Delta_{\bm{Y},\bm{\Omega}}^{\mathrm{rnd}}
\),
continuity of the conditional distribution gives the probability integral transform
\[
  G_{\Delta_{\bm{Y},\bm{\Omega}}^{\mathrm{rnd}}}^{\mathrm{sel,rnd}}
  \!\left(T(\bm{Y})\right)
  \sim
  \mathrm{Unif}(0,1)
  \qquad
  \text{conditionally on }
  \left(
    \widetilde{\mathcal{E}}_{\widetilde{\bm{Y}}},
    \widetilde{\mathcal{Q}}_{\widetilde{\bm{Y}}}
  \right),
\]
which yields \eqref{eq:appendix_randomized_ci_validity}.
Averaging over
\(
(\widetilde{\mathcal{E}}_{\widetilde{\bm{Y}}},\widetilde{\mathcal{Q}}_{\widetilde{\bm{Y}}})
\)
gives \eqref{eq:appendix_randomized_ci_marginal}.
\end{proof}

\newpage
\section{Experimental Details}

\subsection{Detailed Setup}
\label{app:appendix_setup_details}

This subsection provides the experimental details omitted from \S~\ref{sec:experiments}.

\paragraph{Candidate grid and region family.}
For each dimension $d\in\{1,2,3\}$, we use the axis-aligned grid
\(
\mathcal{X}=\{0,\tfrac{1}{m_d-1},\ldots,1\}^d
\)
with
\[
  m_d := \operatorname{round}\!\left(m_1^{1/d}\right),
  \qquad m_1=1024.
\]
For the high-region-versus-low-region target, we form the family $\mathfrak{C}_\ell=\{\mathcal{C}(\bm{s};\ell)\subset[N]:\ \mathcal{C}(\bm{s};\ell)\neq\emptyset,\ \bm{s}\in\R^d\}$ of axis-aligned hypercubes with side length
\[
  \ell_d := \ell_1^{1/d},
  \qquad \ell_1=0.2,
\]
and then construct the selected pair $(\mathcal{I}_{\bm{y}},\mathcal{J}_{\bm{y}})$ via \eqref{eq:region_selector}.

\paragraph{Default ADC hyperparameters.}
For GP-UCB, the default exploration parameter is $\kappa=2$, and the surrogate is GP regression with an RBF kernel of variance $1$ and dimension-adjusted length scale
\[
  \lambda_{\mathrm{RBF},d} := \lambda_{\mathrm{RBF},1}\sqrt{d},
  \qquad \lambda_{\mathrm{RBF},1}=0.1,
\]
with noise variance matched to the response model.
For TPE, the default quantile level is $\gamma=0.2$ and the Gaussian-kernel bandwidth is $0.1$.
Across all methods, tie-breaking and any auxiliary randomization such as choice of initial data points are fixed in advance.

\paragraph{Computational resources.}
We executed all experiments on an AMD EPYC 9474F processor with a 48-core 3.6\,GHz CPU and 768\,GB of memory.

\subsection{Synthetic Objective Families for the Power Experiment}
\label{app:synthetic_objective_families}

This subsection provides the six synthetic objective families used in Figure~\ref{fig:power}.
Each family is constructed from a separable one-dimensional template aggregated across coordinates, followed by rescaling to the range $[-1,1]$ over the finite candidate set $\mathcal{X}$.

\paragraph{Six synthetic objective families and the amplitude parameter $a$.}
Let $u\in[0,1]$ and define the six base one-dimensional functions
\begin{align}
  g_{\mathrm{sinc}}(u)
  &= \operatorname{sinc}\bigl(10(u-\tfrac{1}{2})\bigr),
  &\operatorname{sinc}(t)&=\frac{\sin(\pi t)}{\pi t}\ \ (\operatorname{sinc}(0)=1),
  \\
  g_{\mathrm{cos}}(u)
  &= -\cos(2\pi u),
  \\
  g_{\mathrm{chirp}}(u)
  &= \sin(2\pi u^2),
  \\
  g_{\mathrm{bump}}(u)
  &= \exp\!\left(-\frac{(u-0.7)^2}{2\cdot 0.08^2}\right),
  \\
  g_{\mathrm{peak}}(u)
  &= 1-\lvert u-0.4\rvert,
  \\
  g_{\mathrm{negFor}}(u)
  &= -(6u-2)^2\,\sin(12u-4).
\end{align}
For $\bm{x}=(x_1,\ldots,x_d)\in[0,1]^d$, define the coordinate-aggregated raw objective
\begin{equation}
  \tilde f_k(\bm{x})
  =
  \frac{1}{d}\sum_{j=1}^d g_k(x_j),
  \qquad
  k\in\{\mathrm{sinc},\mathrm{cos},\mathrm{chirp},\mathrm{bump},\mathrm{peak},\mathrm{negFor}\}.
  \label{eq:appendix_power_families_raw}
\end{equation}
We then scale each family to the range $[-1,1]$ over the discrete candidate set $\mathcal{X}$:
\begin{equation}
  m_k=\min_{\bm{x}\in\mathcal{X}}\tilde f_k(\bm{x}),
  \qquad
  M_k=\max_{\bm{x}\in\mathcal{X}}\tilde f_k(\bm{x}),
  \qquad
  s_k(\bm{x})=\frac{2\tilde f_k(\bm{x})-(M_k+m_k)}{M_k-m_k}.
  \label{eq:appendix_power_families_scaled}
\end{equation}
Finally, the amplitude parameter $a$ enters multiplicatively as
\begin{equation}
  f_{\star}^{(k,a)}(\bm{x}) = a\, s_k(\bm{x}).
  \label{eq:appendix_power_families_amplitude}
\end{equation}
In particular, $a=0$ yields the global null $f_\star\equiv 0$.

\subsection{Confidence Interval Length}
\label{app:appendix_ci_length}

This subsection provides the interval-length analysis that complements the coverage results in \S~\ref{subsec:experiments_coverage}.
The setting is the same as in \S~\ref{subsec:experiments_coverage}: the global null $f_\star\equiv 0$ with $d=3$ and varying horizon $N_{\mathrm{steps}}$.
Figure~\ref{fig:ci_len} shows the boxplots of interval lengths.
Across the displayed settings, the intervals of \proposed are wider than those of \naive, which is consistent with the additional conditioning required for validity.
Note that this result does not mean that \proposed is worse than \naive, because \naive is anti-conservative and hence does not achieve the nominal coverage level (see Figure~\ref{fig:coverage}).
\bonferroni intervals are dramatically wider, which aligns with its extreme conservativeness in the type-I error experiment.

\begin{figure}[t]
    \centering
    \setlength{\tabcolsep}{0pt}
    \begin{tabular}{@{}c@{\hspace{0.02\textwidth}}c@{}}
      \includegraphics[page=1,width=0.47\textwidth, trim=0 0.0cm 0 0.70cm, clip]{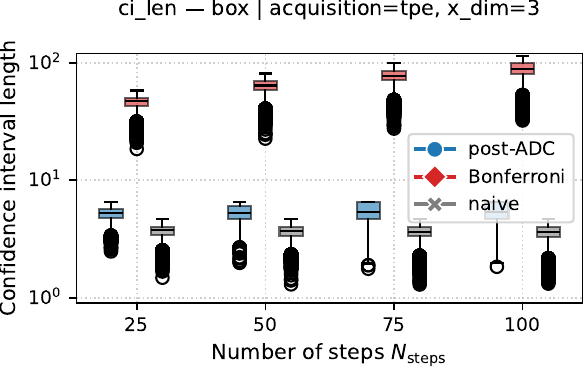} &
      \includegraphics[page=2,width=0.47\textwidth, trim=0 0.0cm 0 0.70cm, clip]{experiments/ci_len_nsteps.pdf} \\[-2mm]
      {\scriptsize \texttt{TPE}} &
      {\scriptsize \texttt{GP-UCB}}
    \end{tabular}
  \caption{Distribution of selective confidence interval lengths (log scale) for $d=3$ across horizons and acquisition rules.}
  \label{fig:ci_len}
\end{figure}

\subsection{Type-I Error vs. Dimension}
\label{app:appendix_fpr_dimension}

This subsection complements Figure~\ref{fig:fpr} by fixing $N_{\mathrm{steps}}=50$ and varying the dimension over $d\in\{1,2,3\}$ under the same global-null model $f_\star\equiv 0$.
Figure~\ref{fig:fpr_dim} reports the resulting empirical rejection probability for GP-UCB and TPE.

\begin{figure}[t]
  \centering
  \setlength{\tabcolsep}{0pt}
  \begin{tabular}{@{}c@{\hspace{0.02\textwidth}}c@{}}
    \includegraphics[page=1,width=0.47\textwidth, trim=0 0.0cm 0 0.7cm, clip]{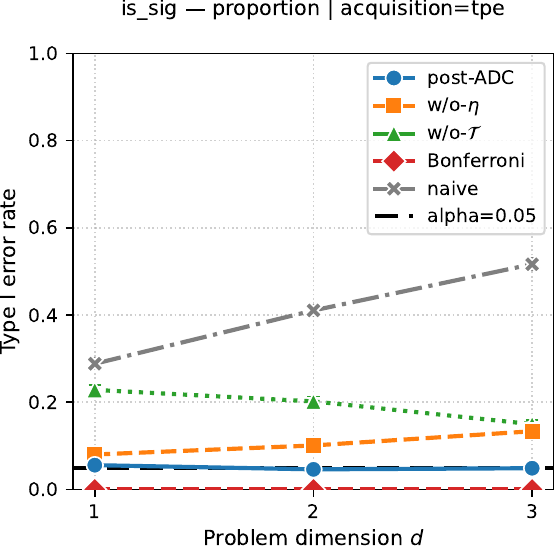} &
    \includegraphics[page=2,width=0.47\textwidth, trim=0 0.0cm 0 0.7cm, clip]{experiments/fpr_dim.pdf} \\[-2mm]
    {\scriptsize \texttt{TPE}} &
    {\scriptsize \texttt{GP-UCB}}
  \end{tabular}
  \caption{Type-I error under the global null $f_\star\equiv 0$ as the problem dimension varies over $d\in\{1,2,3\}$, with $N_{\mathrm{steps}}=50$ fixed. The horizontal dashed line marks the nominal level $\alpha=0.05$.}
  \label{fig:fpr_dim}
\end{figure}

\proposed remains close to the nominal level across the displayed dimensions, while the naive and partially conditioned baselines remain anti-conservative.
As in the main-text null experiment, \bonferroni stays overly conservative throughout this sweep.

\subsection{Type-I Error vs. Acquisition Hyperparameters}
\label{app:appendix_fpr_hyperparameter}

We also examine how the null rejection probability changes when the acquisition-rule hyperparameters are varied while holding the remaining experimental conditions fixed at $d=3$ and $N_{\mathrm{steps}}=50$.
For GP-UCB, we sweep the exploration parameter $\kappa$ over the grid from $0.5$ to $1024$.
For TPE, we sweep the quantile level $\gamma$ over the grid $\{0.1,0.2,\ldots,0.8\}$.
These denser sensitivity sweeps use $1{,}000$ Monte Carlo replicates per setting.
Figure~\ref{fig:hyperparameter} summarizes the resulting type-I error curves.
Across these sweeps, \proposed remains close to the nominal level, while the \naive and partially conditioned baselines remain anti-conservative over broad portions of the hyperparameter ranges.

\begin{figure}[t]
  \centering
  \setlength{\tabcolsep}{0pt}
  \begin{tabular}{@{}c@{\hspace{0.02\textwidth}}c@{}}
    \includegraphics[width=0.47\textwidth, trim=0 0.0cm 0 0.8cm, clip]{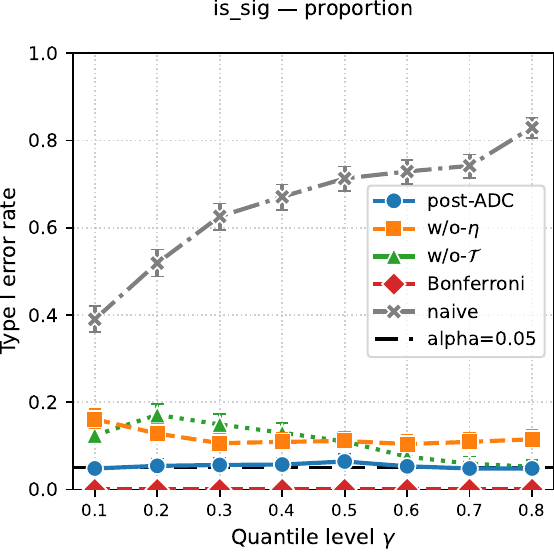} &
    \includegraphics[width=0.47\textwidth, trim=0 0.0cm 0 0.8cm, clip]{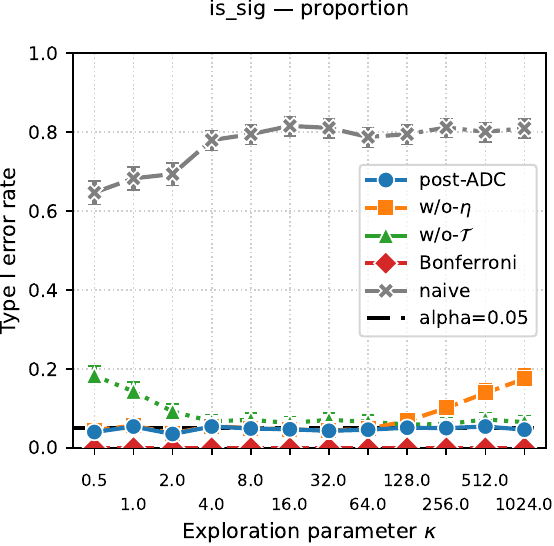} \\[-2mm]
    {\scriptsize \texttt{TPE}} &
    {\scriptsize \texttt{GP-UCB}}
  \end{tabular}
  \caption{Type-I error under the global null $f_\star\equiv 0$ as the acquisition hyperparameters vary, with $d=3$ and $N_{\mathrm{steps}}=50$ fixed. The horizontal dashed line marks the nominal level $\alpha=0.05$.}
  \label{fig:hyperparameter}
\end{figure}

\subsection{Real-Data Experiment Details}
\label{app:appendix_real_data}

\paragraph{Datasets.}
We use four UCI regression benchmarks, summarised in Table~\ref{tab:real_datasets}.
All datasets are licensed under the CC BY 4.0 license.
For each dataset we project the inputs onto a $d$-dimensional subspace spanned by the top-$d$ features ranked by mutual information with the response; feature sets for each $d\in\{1,2,3\}$ are listed in Table~\ref{tab:real_datasets}.
The candidate set $\mathcal{X}$ is formed by taking all distinct projected input vectors in the dataset and, if their count exceeds $1024$, sub-sampling $1024$ uniformly at random.
Response noise variance $\sigma^2$ is estimated by maximising the GP marginal likelihood (RBF kernel, length scale fitted jointly) on the full dataset; this $\hat\sigma^2$ is then used as the fixed noise level in the inference procedure.

\begin{table}[t]
  \centering
  \caption{Real datasets used in \S~\ref{subsec:experiments_real}. Features are listed in order of mutual-information rank; the top $d$ are used for each value of $d$.}
  \label{tab:real_datasets}
  \small
  \setlength{\tabcolsep}{2pt}
  \begin{tabular}{@{}p{0.27\textwidth}rrp{0.20\textwidth}p{0.20\textwidth}@{}}
    \toprule
    Dataset & $N$ & $|\mathcal{X}|$ & Response & \makecell[l]{Top-3 features\\(rank order)} \\
    \midrule
    Concrete Compressive Strength & 1030  & 412  & compressive strength & water, cement, superplasticizer\\
    Gas Turbine (CO)              & 36733 & 1024 & CO emission          & TIT, TEY, CDP \\
    Gas Turbine (NOx)             & 36733 & 1024 & NOx emission         & TIT, AT, TEY \\
    Combined Cycle Power Plant    & 9568  & 1024 & net hourly output    & AT, V, AP \\
    \bottomrule
  \end{tabular}
\end{table}

\paragraph{Experimental protocol.}
For each dataset and each $(d, \text{acquisition})$ combination, we draw $1{,}000$ independent bootstrap replicates.
In each replicate, $N = N_{\mathrm{init}} + N_{\mathrm{steps}} = 10 + 50 = 60$ samples are drawn uniformly without replacement from the full dataset and used as the observed responses.
The regions $(\mathcal{I}_{\bm{y}}, \mathcal{J}_{\bm{y}})$ are constructed exactly as in the synthetic experiments.
We compare only \proposed and \bonferroni.

\paragraph{Results for $d=2$ and $d=3$.}
Tables~\ref{tab:real_d2} and \ref{tab:real_d3} report the empirical rejection probabilities for $d=2$ and $d=3$, respectively.
The results are consistent with the $d=1$ results in Table~\ref{tab:real_d1}: \proposed detects signal on datasets with informative structure, while \bonferroni remains essentially powerless throughout.

\begin{table}[t]
  \centering
  \caption{Empirical rejection probability on five real datasets with $d=2$ and $1{,}000$ bootstrap replicates.}
  \label{tab:real_d2}
  \small
  \setlength{\tabcolsep}{3pt}
  \begin{tabular}{@{}llcccc@{}}
  \toprule
  Acquisition & Method & Concrete & \makecell{Gas Turbine\\(CO)} & \makecell{Gas Turbine\\(NOx)} & \makecell{Power\\Plant} \\
  \midrule
  \multirow{2}{*}{ucb} & \proposed & \textbf{0.145} & \textbf{0.994} & \textbf{0.836} & \textbf{1.000} \\
  & \bonferroni & 0.000 & 0.000 & 0.000 & 0.000 \\
  \midrule
  \multirow{2}{*}{tpe} & \proposed & \textbf{0.246} & \textbf{0.952} & \textbf{0.781} & \textbf{0.992} \\
  & \bonferroni & 0.000 & 0.000 & 0.000 & 0.000 \\
  \bottomrule
  \end{tabular}
\end{table}

\begin{table}[t]
  \centering
  \caption{Empirical rejection probability on five real datasets with $d=3$ and $1{,}000$ bootstrap replicates.}
  \label{tab:real_d3}
  \small
  \setlength{\tabcolsep}{3pt}
  \begin{tabular}{@{}llcccc@{}}
  \toprule
  Acquisition & Method & Concrete & \makecell{Gas Turbine\\(CO)} & \makecell{Gas Turbine\\(NOx)} & \makecell{Power\\Plant} \\
  \midrule
  \multirow{2}{*}{ucb} & \proposed & \textbf{0.166} & \textbf{0.999} & \textbf{0.741} & \textbf{1.000} \\
  & \bonferroni & 0.000 & 0.003 & 0.000 & 0.000 \\
  \midrule
  \multirow{2}{*}{tpe} & \proposed & \textbf{0.167} & \textbf{0.975} & \textbf{0.747} & \textbf{0.999} \\
  & \bonferroni & 0.000 & 0.006 & 0.000 & 0.000 \\
  \bottomrule
  \end{tabular}
\end{table}

\clearpage
\bibliographystyle{unsrt}
\bibliography{ref}

\end{document}